\documentclass{article}

\usepackage[square,sort,comma,numbers]{natbib}
\usepackage[preprint]{neurips_2020}

\usepackage[utf8]{inputenc} 
\usepackage[T1]{fontenc}    
\usepackage{hyperref}       
\usepackage{url}            
\usepackage{booktabs}       
\usepackage{amsfonts}       
\usepackage{nicefrac}       
\usepackage{microtype}      
\usepackage{wrapfig}

\usepackage{algorithm}
\usepackage{algorithmic}
\usepackage{graphicx}
\usepackage{booktabs,xcolor} 

\usepackage{amsmath,amsfonts,comment,bbm}

\newcommand{\cF}{\mathcal{F}}

\newcommand{\up}{^{\prime}}

\newcommand{\cB}{\mathcal{B}}

\newcommand{\cA}{\mathcal{A}}
\newcommand{\cS}{\mathcal{S}}
\newcommand{\ust}{^{\star}}

\newcommand{\cE}{\mathcal{E}}
\newcommand{\cM}{\mathcal{M}}

\newcommand{\cI}{\mathcal{I}}

\newtheorem{lemma}{Lemma}
\newtheorem{proof}{Proof}
\newtheorem{remark}{Remark}
\newtheorem{theorem}{Theorem}

\def\cA{{\mathcal{A}}} \def\cB{{\mathcal{B}}}  
\def\cE{{\mathcal{E}}} \def\cF{{\mathcal{F}}}  
\def\cI{{\mathcal{I}}}   
\def\cM{{\mathcal{M}}}  \def\cO{{\mathcal{O}}} \def\cP{{\mathcal{P}}}
 \def\cR{{\mathcal{R}}} \def\cS{{\mathcal{S}}}

   \def\bs{{\mathbf{s}}}

\newcommand{\bef}{\begin{figure}}
\newcommand{\eef}{\end{figure}}
\newcommand{\beq}{\begin{eqnarray}}
\newcommand{\eeq}{\end{eqnarray}}

\title{Reinforcement Learning for Finite-Horizon Restless Multi-Armed Multi-Action Bandits}

\author{
Guojun Xiong\\
SUNY-Binghamton University\\
Binghamton, NY 13902\\
 \texttt{gxiong1@binghamton.edu}
  \And
Jian Li\\
SUNY-Binghamton University\\ 
Binghamton, NY 13902\\
 \texttt{lij@binghamton.edu}
  \And
  Rahul Singh\\
Indian Institute of Science\\
Bengaluru, Karnataka 560012, India\\
 \texttt{rahulsingh@iisc.ac.in}
}

\begin{document}

\maketitle

\begin{abstract}
We study a finite-horizon restless multi-armed bandit  problem with multiple actions, dubbed R$(\text{MA})^2$B.  The state of each arm evolves according to a controlled Markov decision process (MDP), and the reward of pulling an arm depends on both the current state of the corresponding MDP and the action taken.  The goal is to sequentially choose actions for arms so as to maximize the expected value of the cumulative rewards collected. Since finding the optimal policy is typically intractable, we propose a computationally appealing index policy which we call  \textit{Occupancy-Measured-Reward Index Policy}.  Our policy is well-defined even if the underlying MDPs are not indexable. We prove that it is asymptotically optimal when the activation budget and number of arms are scaled up, while keeping their ratio as a constant. For the case when the system parameters are unknown, we develop a learning algorithm. Our learning algorithm uses the principle of optimism in the face of uncertainty and further uses a generative model in order to fully exploit the structure of \textit{Occupancy-Measured-Reward Index Policy}. We call it the R$(\text{MA})^2$B-UCB algorithm. As compared with the existing algorithms, R$(\text{MA})^2$B-UCB performs close to an offline optimum policy, and also achieves a sub-linear regret with a low computational complexity.  Experimental results show that R$(\text{MA})^2$B-UCB outperforms the existing algorithms in both regret and run time.
\end{abstract}

\section{Introduction}\label{sec:intro}

We study a variant of the popular restless multi-armed problem (RMAB)~\cite{whittle1988restless} in which the decision maker has to make choices for only a finite horizon, and can choose from amongst multiple actions for each arm. We call this problem as the restless multi-armed multi-action bandits (R$(\text{MA})^2$B).  A RMAB problem requires a decision maker to choose from amongst a fixed number of competing ``arms'' in a sequential manner. Each arm is endowed with a ``state'' that evolves according to a Markov decision process (MDP)~\cite{puterman1994markov} that is independent of other arms. In the multi-armed bandit (MAB) problem~\cite{gittins1974dynamic}, i.e. the ``rested MAB'' or simply the MAB, states of only those arms evolve that are activated currently, and rewards are generated only from these arms. The goal is to maximize the expected value of the cumulative rewards collected, by choosing the arms in a sequential way. The celebrated \textit{Gittins index policy} \cite{gittins1974dynamic} yields an efficient solution to the MAB. At each time, it assigns an index to each arm, which is a function of the current state of this arm, and then activates the arm with the largest index. However, Gittins index policy is optimal only when the following assumptions hold (i) the MAB problem involves rested bandits; (ii) \textit{only one arm} can be activated at each decision epoch; and (iii) the objective is \textit{infinite-horizon discounted} expected reward.  Whittle  \cite{whittle1988restless} generalized Gittins policy to also allow for the evolution of those arms that are not activated currently (dubbed as ``a changing world setting''), thereby introducing the RMAB problem.  Whittle's setup also allows multiple arms to be activated simultaneously.

The RMAB problem is a very general setup that can be applied to solve a variety of sequential decision making problems ranging from job allocation \cite{nino2007dynamic,jacko2010restless,bertsimas2000restless},  
wireless communication \cite{dai2011non,sheng2014data}, sensor management \cite{mahajan2008multi,ahmad2009optimality} and healthcare \cite{deo2013improving,lee2019optimal,mate2021risk,killian2021beyond}.  However, the RMAB is notoriously \textit{intractable} \cite{papadimitriou1994complexity} and the optimal policy for an RMAB is rarely an index policy.     
To that end, Whittle 
proposed a heuristic policy for the \textit{infinite-horizon} RMAB, which is now called the \textit{Whittle index policy}. 
However, Whittle index policy is well-defined only when the so-called \textit{indexability}~\cite{whittle1988restless} condition is satisfied.  Furthermore, even when an arm is indexable, obtaining its Whittle indices could still be intractable, especially when the corresponding controlled Markov process is convoluted \cite{nino2007dynamic}. Finally, Whittle index policy is \textit{only guaranteed} to be asymptotically optimal 
\cite{weber1990index} under a difficult-to-verify condition which requires that the fluid approximation associated with the ``population dynamics'' of the overall system has a globally asymptotically stable attractor.

Inspired by Whittle's work, many studies focused on finding the index policy for restless bandit problems, e.g., \cite{nino2001restless,verloop2016asymptotically,hu2017asymptotically,zayas2019asymptotically,brown2020index}.  This line of works assumes that the system parameters are known to the decision-maker.~Since in reality the true value of the parameters are unavailable, and possibly time-varying, it is important to examine RMAB from a learning perspective, e.g., \cite{dai2011non,liu2011logarithmic,tekin2011adaptive,liu2012learning,tekin2012online,ortner2012regret,jung2019regret,jung2019thompson,wang2020restless,xiong2021learning}. However, analyzing learning algorithms for RMAB is in general hard due to the uncertainty associated with the learner's knowledge about the system parameters, and secondly since the design of optimal control policy even when the parameter is known, is still unresolved.

Firstly, the existing algorithms such as~\cite{liu2011logarithmic,tekin2011adaptive,tekin2012online,liu2012learning} that are based on the upper confidence bound (UCB) strategy~\cite{auer2002finite} may not perform close to the offline optimum. This is the case because \textit{the baseline policy} in these works is often a heuristic policy that does not have any theoretical performance guarantees. An example of such a heuristic policy is one that pulls only one arm, or a fixed set of arms. Such policies are known to yield highly sub-optimal performance in the RMAB setting, and this makes the $\cO(\log T)$ learning regret~\cite{lattimore2020bandit} less meaningful. Secondly, the aforementioned learning algorithms with a theoretical guarantee of an $\tilde{\cO}(\sqrt{T})$ regret are often \textit{computationally expensive}. For example, the colored-UCRL2 algorithm~\cite{ortner2012regret} suffers from {an exponential computational complexity}, and the regret bound is {exponential in the number of states and arms}.  This is because it needs to solve Bellman equations on a state-space that has a size which grows exponentially with the number of arms.~Thirdly, existing low-complexity policies such as \cite{wang2020restless,xiong2021learning} often do not have a regret guarantee that scales as \textit{$\tilde{\cO}(\sqrt{T})$}, and moreover these also restrict to a specific Markovian model, that are hard to generalize.  In a different line of works, Thompson sampling based algorithms \cite{jung2019regret,jung2019thompson} were used to solve this problem. These provide a theoretical guarantee in the Bayesian setup, but since very often the likelihood functions are complex, these are required to implement a computationally expensive method to update the posterior beliefs.~To the best of our knowledge, there are no provably optimal policies for RMAB problems (let alone the R$(\text{MA})^2$B in consideration) with an efficient learning algorithm that performs \textit{close to the offline optimum} and achieves \textit{a sub-linear regret} and with \textit{a low computational complexity}, all at once.

In this paper, we address the above challenges for R$(\text{MA})^2$B problems that involve operating over a finite time horizon.  In contrast to most of the aforementioned existing literature, that allow for only a binary decision set (activate or not activate), our setup allows the decision maker to choose from multiple actions for each arm. This is a very useful generalization since many applications are not limited to binary actions.  For example, while performing video streaming by transmitting video data packets across wireless channels, the transmitter can dynamically choose varying levels of transmission power or video resolution (i.e., actions), which in turn affects the quality of video streaming experienced by the users. However, the analysis of restless bandits with multiple actions largely remains elusive for the general setting in the literature.  We make progress toward R$(\text{MA})^2$B problems by making the following contributions:

\noindent$\bullet$ \textbf{Asymptotically optimal index policy.} 
For the general finite-horizon R$(\text{MA})^2$B problems in which the system parameters are known, we propose an index policy which we call \textit{Occupancy-Measured-Reward Index Policy}.  We show that our policy is asymptotically optimal, a result paralleling those known for the Whittle policy. 
However, unlike Whittle index policy, our index policy does not require the indexability condition to hold, and is well-defined for both indexable and nonindexable R$(\text{MA})^2$B problems.  This result is significant since the indexability condition is hard to verify or may not hold true in general, and the non-indexable settings have so far received little attention, even though they arise in many practical problems.

\noindent$\bullet$ \textbf{Reinforcement learning augmented index policy.} 
We present one of the first generative model based reinforcement learning augmented algorithm toward an index policy in the context of finite-horizon R$(\text{MA})^2$B problems. We call our algorithm R$(\text{MA})^2$B-UCB. R$(\text{MA})^2$B-UCB consists of a novel optimistic planning step similar to the UCB strategy, in which it obtains an estimate of the model by sampling state-action pairs in an offline manner and then solves a so-called extended linear programming problem that is posed in terms of certain ``occupancy measures''. The complexity of this procedure is linear in the number of arms, as compared with exponential complexity of the state-of-the-art colored-UCRL2 algorithm.  Furthermore, we show that R$(\text{MA})^2$B-UCB achieves a $\tilde{\cO}(\sqrt{T})$ regret and hence performs close to the offline optimum policy since it contains an efficient exploitation step enabled by the optimistic planning used by our \textit{Occupancy-Measured-Reward Index Policy}.  This significantly outperforms other existing methods~\cite{liu2011logarithmic,tekin2011adaptive,tekin2012online,liu2012learning} that often rely upon a heuristic policy.  Moreover, the multiplicative ``pre-factor" that goes with the time-horizon dependent function in the regret is quite low for our policy since the ``exploitation step'' that we propose is much more efficient, in fact this is ``exponentially better'' than that of the colored-UCRL2.  Our simulation results also show that  R$(\text{MA})^2$B-UCB outperforms existing algorithms in both regret and running time.

\noindent\textbf{Notation.} We denote the set of natural and real numbers by $\mathbb{N}$ and $\mathbb{R}$, respectively.  We let $T$ be the finite number of total decision epochs (time).  We denote the cardinality of a finite set $\cA$ by $A:=|\cA|.$ We also use $[N]$ to represent the set of integers $\{1,\cdots, N\}$ for $N\in\mathbb{N}.$

\section{System Model}\label{sec:model}
We begin by describing the finite-horizon R$(\text{MA})^2$B problem in which the action set for each of the $N$ arms is allowed to be non-binary.  Each arm $n$ is described by a unichain Markov decision process (MDP) \cite{kallenberg2003finite} $(\cS_n, \cA_n, P_n, r_n, \bold{s}_1, T)$, where $\cS_n$ is its finite state space, $\cA_n$ is the set of finite actions, $P_n:\cS_n\times\cA_n\times\cS_n\mapsto\mathbb{R}$ is the transition kernel and $r_n: \cS_n\times\cA_n\mapsto \mathbb{R}$ is the reward function.  For the ease of readability, we assume  that all arms share the same state and action spaces, and these are denoted by $\cS$ and $\cA$, respectively.  Our results and analysis can be extended in a straightforward manner to the case of different state and action spaces, though this will increase the complexity of notation.

Without loss of generality, we denote the action set as $\cA=\{0, 1,\cdots, A\}$ where $A<\infty$.  By using the standard terminology from the RMAB literature, we call an arm \textit{passive} when action $a=0$ is applied to it, and \textit{active} otherwise. An \textit{activation cost} of $a$ units is incurred each time an arm is applied action $a$ (thus not activating the arm $a=0$ corresponds to $0$ activation cost). The total activation cost associated with activating a subset of the $N$ arms at each time $t$ is constrained by $K$ units. The quantity $K$ is called the \textit{activation budget.} The initial state is chosen according to the initial distribution $\bold{s}_1$ and $T<\infty$ is the operating time horizon.  

We denote the state of arm $n$ at time $t$ as $s_n(t)\in \cS$.  The process $s_n(t)$ evolves as a controlled Markov process with the conditional probability distribution of $s_n(t+1)$ given by $P_n(s_n(t),a_n(t), s_n(t+1))$ (almost surely). The instantaneous reward earned at time $t$ by activating  arm $n$ is denoted by a random variable $r_n(t) : = r_n(s_n(t),a_n(t))$.  Without loss of generality,  we assume that $r_n(s_n,a_n)\in[0,1]$, $\forall n$ with expectation $\mathbb{E}r_n(t)=\bar{r}_n(s,a)$ \cite{efroni2020exploration}, and let $r_n(s,0)$ be $0$ $\forall s\in\cS$, i.e., no reward is earned when the arm is passive.  Denote the total reward earned at time $t$ by $R(t)$, i.e., $R(t):= \sum_n r_n(t)$.   
Let $\cF_t$ denote the operational history until $t$, i.e., the sigma-algebra \cite{shiryaev2007optimal} generated by the random variables $ \{s_n(\ell):n\in [N],\ell \in [t] \}, \{a_n(\ell):n\in [N],\ell \in [t-1] \}$. Our goal is to derive a policy $\pi: \cF_t \mapsto \cA^{N},~t=1,2,\ldots$, that makes decisions regarding which set of arms to activate at each time $t\in[T]$, so as to maximize the expected value of the cumulative rewards subject to a budget constraint on the activation resource, i.e., 
\begin{align}\label{obj-cons}
\max_{\pi}~~\mathbb{E}_\pi \left(\sum_{n=1}^{N} \sum_{t=1}^{T} r_n(t)\right)\quad \mbox{s.t.}~\sum_{n=1}^{N} a_n(t)\leq K,\quad\forall t\in [T],
\end{align}
where the subscript indicates that the expectation is taken with respect to the measure induced by the policy $\pi.$  We refer to the problem \eqref{obj-cons} as the ``original problem''. Though this could be solved by using existing techniques such as dynamic programming~\cite{hernandez2012further}, existing approaches  suffer from the ``curse of dimensionality'' \cite{bellmanbook,bertsekas1995dynamic}, and hence are computationally intractable.   
We overcome this difficulty by developing a computationally feasible and provably optimal index-based policy.

\section{Asymptotically Optimal Index Policy}\label{sec:truncated-policy}

In this section, we focus on the scenario when the controlled transition probabilities and the reward functions of each arm are known. We design an index policy for the finite-horizon R$(\text{MA})^2$Bs, and show that it is asymptotically optimal. We begin by introducing a certain ``relaxed problem''~\cite{whittle1988restless}. The relaxed problem can be solved efficiently since it can equivalently be posed as a linear programming (LP) in the space of occupation measures of the $N$ controlled Markov processes~\cite{altman1999constrained}, where each such process corresponds to one arm. This forms the building block of our proposed index-based policy, and is described next.

\subsection{The Relaxed Problem}
Consider the following problem obtained by relaxing the ``hard'' constraint in~\eqref{obj-cons} in which the activation cost at each time $t\in[T]$ is limited by $K$ units, by a ``relaxed'' constraint in which this is supposed to be true only in an expected sense, i.e.,
\begin{align}\label{eq:obj_con_rel}
\max_{\pi}~~ \mathbb{E}_{\pi} \left(\sum_{n=1}^{N} \sum_{t=1}^{T} r_n(t)\right)\quad
\mbox{s.t.}~\mathbb{E}_{\pi}\left\{\sum_{n=1}^{N} a_n(t) \right\}\leq K,\quad \forall t\in[T]. 
\end{align}
Obviously the optimal value of the relaxed problem \eqref{eq:obj_con_rel} yields an upper bound on the optimal value  of~\eqref{obj-cons}. We note that an optimal policy for \eqref{eq:obj_con_rel} might require randomization \cite{altman1999constrained}. It is well known \cite{altman1999constrained} that the relaxed problem \eqref{eq:obj_con_rel} can be reduced to a LP in which the decision variables are the occupation measures of the controlled process.  
More specifically, the occupancy measure $\mu$ of a policy $\pi$ of a finite-horizon MDP describes the probability with which state-action pair $(s,a)$ is visited at time $t$. Formally,
\begin{align*}
    \mu= \left\{\mu_n(s,a;t)=\mathbb{P}(s_n(t)=s, a_n(t)=a): \forall n\in[N], t\in[T]\right\}.
\end{align*}
The relaxed problem~\eqref{eq:obj_con_rel} can be reformulated as the following LP \cite{altman1999constrained} in which the decision variables are these occupation measures: 
\begin{align}
\max_{\mu}&~~\sum_{n=1}^{N}\sum_{t=1}^T\sum_{(s,a)} \mu_n(s,a;t)\bar{r}_n(s,a)\label{eq:obj_rel_lp} \displaybreak[0]\\
\hspace{-0.1cm}\mbox{ s.t. }& ~~{ \sum_{n=1}^{N}\sum_{(s,a)} a\mu_n(s,a;t)\leq K}, \quad\forall t\in[T],
\label{eq:con_rel_lp}\displaybreak[1]\\
&~~ {\sum_{a}} \mu_n(s,a;t)=\sum_{(s^\prime,a^\prime)}\mu_n(s^\prime, a^\prime; t-1)P_n(s^\prime, a^\prime,s),\quad n\in[N], t\in[T],\label{eq:pmc_rel_lp1}\displaybreak[3]\\
&~~\sum_{a}\mu_n(s,a;1)={\bold{s}_1(s)}, 
~\forall s\in\cS,
\label{eq:pmc_rel_lp2}    
\end{align}
where \eqref{eq:con_rel_lp} is a restatement of the constraint in \eqref{eq:obj_con_rel} for $\forall t\in[T]$, which indicates the activation budget; 
\eqref{eq:pmc_rel_lp1} represents the transition of the occupancy measure from time  $t-1$ to time $t$, $\forall n\in [N]$ and $\forall t\in[T]$; and \eqref{eq:pmc_rel_lp2} indicates the initial condition for occupancy measure at time $1$, $\forall s\in\cS$.
From the constraints~(\ref{eq:pmc_rel_lp1})-(\ref{eq:pmc_rel_lp2}), it can be easily checked that the occupancy measure satisfies $\sum_{s,a}\mu_n(s,a,t)=1$, $
\forall t\in[T]$. Thus, the occupancy measure $\mu_n, \forall n\in[N]$ is a probability measure.

An optimal policy for \textit{the relaxed problem} can be obtained from the solution of this LP as follows~\cite{altman1999constrained}. Let $\mu\ust = \left\{\mu\ust_{n}(s,a;t):n\in [N], t\in [T]  \right\}$ be a solution of the above LP.  Construct the following Markovian non-stationary randomized  policy 
$\chi\ust=\{\chi\ust_n(t): n\in[N], t\in[T]\}$ as follows: 
if the state $s_n(t)$ is $s$ at time $t$, then $\chi\ust_{n}{(t)}$ chooses an action $a$ with a probability equal to 
\begin{align}
\chi\ust_n(s,a;t):=
\frac{\mu\ust_n(s,a;t)}{\sum\limits_{a\up\in \cA}\mu\ust_n(s,a\up; t)}.
\label{def:policy}
\end{align}
If the denominator of \eqref{def:policy} equals zero, i.e., state $s$ for arm $n$ is not reachable at time $t$, arm $n$ can be simply made passive, i.e., $\chi\ust_n(s,0;t)=1$ and $\chi\ust_n(s,a;t)=0, \forall a\in\mathcal{A}\setminus \{0\}.$

\subsection{The Occupancy-Measured-Reward Index Policy}
The Markov policy $\chi\ust$ constructed from solutions to the above LP form the building block of our index policy for the original problem \eqref{obj-cons}. Note that the policy~(\ref{def:policy}) is not always feasible for \textit{the original problem} since in the latter at most $K$ units of activation costs can be consumed at any time, while (\ref{def:policy}) could spend more than $K$ units of costs at any given time. To this end, our index policy assigns an index to each arm based on its current state and the current time.  We denote by $\psi_n(s_n(t);t)$ the index associated with arm $n$ at time $t$, \begin{align}\label{eq:occupancy-index}
    \psi_n(s_n(t);t):=\sum_{a\in\mathcal{A}\setminus \{0\}}\chi\ust_n(s_n(t),a;t)\bar{r}_n(s_n(t),a), 
\end{align}
where $\chi\ust_n(s_n(t),a;t)$ is defined in \eqref{def:policy}. We call this the \textit{occupancy-measured-reward index (OMR index)}  since it is based solely upon the optimal occupancy measure derived by solving the LP~\eqref{eq:obj_rel_lp}-\eqref{eq:pmc_rel_lp2} and the mean reward, representing the expected obtained reward for arm $n$ at state $s_n(t)$ of time $t$.   Let $\psi(t):=\{\psi_n(s_n(t);t): n\in[N]\}$ be the OMR indices associated with $N$ arms at time $t$. 
Let $a\ust_n(s_n(t);t)$ be the action for arm $n$ in its current state $s_n(t)$ at time $t$, and let $\mathcal{B}(t)$ be the set of arms that are active arms at time $t$.  Our index policy then activates arms with OMR indices in a decreasing order.  
The choice of $\cB(t)$ satisfies the
constraint $\sum_{n\in\cB(t)} a\ust_n(s_n(t);t)\leq K$. The remaining arms $[N]\setminus\mathcal{B}(t)$ are kept passive at time $t$. For each arm that has been chosen to be activated, the action applied to it is selected randomly according to the probability  $\chi\ust_n(s_n(t),a;t)$~\eqref{def:policy}. When multiple arms sharing the same OMR indices,  we randomly activate one arm and allocate the remaining activation costs across all possible actions according to the probability $\chi\ust_n(s_n(t),a;t)$.  If it happens that the indices of all the arms are zero, then all the remaining arms are made passive. 
We call this \textit {an Occupancy-Measured-Reward Index Policy (\textit{OMR Index Policy})}, and denote it as $\pi\ust=\{\pi_n\ust, n\in[N]\}$, which is summarized in Algorithm \ref{alg:Trun_IP}.

\begin{algorithm}
	\caption{\textit{OMR Index Policy}}
	\label{alg:Trun_IP}
\textbf{Input}: 
	Initialize $s_n(1)$ and $\cB(t)$ as an empty set $\forall n\in[N],t\in[T].$
	\begin{algorithmic}[1]
	 	\STATE Construct the LP according to \eqref{eq:obj_rel_lp}-\eqref{eq:pmc_rel_lp2} 
	 	and solve the occupancy measure $\mu$; 
			\STATE Compute $\chi_n^\star(s,a,t), { \forall s,} a, t$ according to  \eqref{def:policy};
		\STATE Construct the index set $\psi(t):=\{\psi_n(s_n(t);t): n\in[N]\}$; and sort $\psi(t)$ in a decreasing order; \label{step:algo}
		\WHILE{$\sum_{n\in\cB(t)}a\ust_n(s_n(t);t)\leq K$}
			\STATE Activate arms according to the order in step \ref{step:algo} and randomly generate a feasible action according to the distribution $\chi_n^\star(s,a,t)$. Store the newly activated arm $n$ into $\cB(t)$;
		\ENDWHILE
	\end{algorithmic}
\end{algorithm}

\begin{remark}
Our index policy is computationally appealing since it is based only on the ``relaxed problem'' by solving a LP.  Furthermore, if all arms share the same MDP, the LP can be decomposed across arms as in \cite{whittle1988restless}, and hence the computational complexity does not scale with the number of arms.  Even more importantly, our index policy is well-defined even if the underlying MDPs are not indexable~\cite{whittle1988restless}. This is in contrast to most of the existing Whittle index-based policies that are only well defined in the case that the system is indexable, which is hard to verify and may not hold in general.  A line of works \cite{verloop2016asymptotically,hu2017asymptotically,zhang2021restless} have been focusing on designing index policies without the indexability requirement, and closest to our work is the parallel work on restless bandits \cite{zhang2021restless} with known transition probabilities and reward functions.  In particular, \cite{zhang2021restless} explores index policies that are similar to ours, but under the assumption that the individual MDPs of each arms are homogeneous. They consider the binary action setup, and focus mainly on characterizing the asymptotic sub-optimality gap.  Our index policy in this section can be seen as the complement to it by considering the general case of heterogeneous MDPs in which multiple actions are allowed for each arm.  Finally, reinforcement learning augmented index policy and the regret analysis in next section also distinguishes our work.  
\end{remark}

\subsection{Asymptotic Optimality}
For the abuse of notation, we let the number of arms be $\rho N$ and the value of activation constraint be $\rho K$ in the limit with $\rho\rightarrow\infty.$ In other words, it represents the scenarios where there are $N$ different classes of arms and each class contains $\rho$ arms. 
Our \textit{OMR Index Policy} achieves asymptotic optimality when the number of arms $\rho N$ and the activation constraint $\rho K$
go to infinity while holding $\alpha=K/N$ constant\footnote{We consider the asymptotic optimality in the same limit as by Whittle \cite{whittle1988restless} and others \cite{weber1990index,verloop2016asymptotically,hu2017asymptotically,zayas2019asymptotically}.}.  Let $R(\pi, \rho K, \rho N)$ denote the expected reward of the original problem~(\ref{obj-cons}) obtained by an arbitrary policy $\pi$ in this limit. 
Denote the optimal policy of the original problem~(\ref{obj-cons}) as $\pi^{opt}$. 

\begin{theorem}\label{thm:asym_opt}
The \textit{OMR Index Policy} achieves the asymptotic optimality as follows
\begin{align*}
\lim_{\rho\rightarrow \infty} \frac{1}{\rho}\Big(R(\pi^\star, \rho K, \rho N)- R(\pi^{opt}, \rho K, \rho N)\Big)=0.
\end{align*}
\end{theorem}
\begin{remark}
Theorem \ref{thm:asym_opt} indicates that as the number of per-class arms (i.e., $\rho$) goes to infinity, the gap between the performance achieved by our \textit{OMR Index Policy} and the optimal policy $\pi^{opt}$ is bounded, and thus per arm gap tends to be zero. 
\end{remark}

\section{Reinforcement Learning for the \textit{OMR Index Policy} }\label{sec:learning}

Computing the \textit{OMR Index Policy} requires the knowledge of the controlled transition probabilities and the reward functions associated with the MDPs of each arm. Since these quantities are typically unknown, we propose a generative model based reinforcement learning (RL) augmented algorithm that learns this policy.

\subsection{The Learning Problem}
The setup is similar to the finite-horizon R(MA)$^2$B described earlier, in which each arm is associated with a controlled MDP $(\cS, \cA, P_n, r_n, \bold{s}_1, T)$.  The only difference is that now the agent does not know the quantities $P_n,r_n$. To judge the performance of the learning algorithm, we use the popular metric of learning regret~\cite{lattimore2020bandit}.  Let $\xi(\pi, \bs_1):=\lim_{T\rightarrow \infty}\frac{1}{T}\mathbb{E}[R(\pi, \bs_1, T)],$ be the average value of expected rewards, and denote the optimal average reward rate by $\xi^{opt}:=\sup_\pi \xi(\pi,\bs_1)$. Note that the optimal average reward rate is independent of the initial state for MDPs that have a finite diameter \cite{puterman1994markov}.

The regret of a policy $\pi$ is defined as follows,
$$
\Delta(\pi,\bs_1, T):=T\xi^{opt}-\mathbb{E}_\pi[R(\pi, \bs_1, T)],
$$
where
$$
R(\pi, \bs_1, T):=\sum_{t=1}^{T}r(t),
$$
is the cumulative rewards collected when the system begins in state $\bs_1$. Thus, regret measures the difference between the rewards collected by the learning policy, and the optimal stationary policy that could be implemented if the system parameters were known to the agent.

\subsection{A Generative Model Based Learning Algorithm}
Our proposed RL algorithm is based on the UCB strategy~\cite{auer2002finite,lattimore2020bandit}, and also uses a generative model similar to \cite{kearns2002sparse}. We call our RL algorithm as R(MA)$^2$B-UCB policy, and depict it in Algorithm~\ref{alg:UCB}.

There are two phases in  R(MA)$^2$B-UCB: (i) a planning phase, and (ii) a policy execution phase.  The planning phase (lines 1-6 in Algorithm~\ref{alg:UCB}) constructs a confidence ball that contains a set of plausible MDPs for each of the $N$ arms. Specifically, we explore a generative approach with a single step simulator that can generate samples of the next state and reward given any state and action 
\cite{kearns2002sparse,hasanzadezonuzy2020learning}. 
It then obtains an optimistic estimate of the true MDP parameters by solving an optimistic planning problem in which the agent can choose MDP parameters from the confidence ball.  This problem can be posed as a LP in which the decision variables are the occupancy measures corresponding to the processes associated with $N$ arms. We can then define the corresponding \textit{OMR Index Policy}.  The planning problem, referred to as an \textit{extended LP} in Algorithm~\ref{alg:UCB} is described below.  Our key contribution here is to choose the right value of $\Lambda(T)$ to balance the accuracy and complexity, which contributes to the properties of sub-linear regret and low-complexity of R(MA)$^2$B-UCB. 

At the policy execution phase (line 7 in Algorithm~\ref{alg:UCB}), the derived \textit{OMR Index Policy} is executed. 
Our key contribution here is to leverage our proposed \textit{OMR Index Policy}, rather than using heuristic ones as in existing algorithms. Since our proposed \textit{OMR Index Policy} is near-optimal, this guarantees that R(MA)$^2$B-UCB performs close to the offline optimum. Moreover, this contributes to the low multiplicative ``pre-factor'' that goes with the time-horizon dependent function in the regret. The prefactor of our algorithm is exponentially better than that of the state-of-the-art colored-UCRL2.

\begin{algorithm}[t]
	\caption{ R$(\text{MA})^2$B-UCB Policy  }
	\label{alg:UCB}
	\textbf{Input}: Learning horizon $T$, and
	learning counts $\Lambda(T)<T$.   
	\begin{algorithmic}[1]
		\FOR{$n=1,2,...,N$ and $(s,a)\in\cS\times\cA$}
		\STATE Sample pairs $(s,a)$ of arm $n$ for $\Lambda(T)$ times. 
		\ENDFOR
		\STATE Construct $ \cP_n(s,a)$ and $\cR_n(s,a)$ according to \eqref{eq:confidence_ball}; 
			\STATE Compute the optimal solution of the extended LP \eqref{eq:UCB_extended};
		\STATE Establish the corresponding \textit{OMR Index Policy} $\pi^\star$;
			\STATE Execute $\pi^\star$ for the rest of the game. 
	\end{algorithmic}
\end{algorithm}

\noindent\textbf{Optimistic planning.} We sample each state-action pair of arm $n$ for $\Lambda(T)$ (the value of $\Lambda(T)$ will be specified later) number of times uniformly across all state-action pairs.  We denote the number of times that a transition tuple $(s,a,s^\prime)$ was observed within $\Lambda(T)$ as $T_n(s,a,s^\prime),$ satisfying 
\begin{align*}
    T_n(s,a, s^\prime)=\sum_{h=1}^{\Lambda(T)}\pmb{1}(s_n({h+1})=s^\prime|s_n({h})=s,a_n({h})=a), \quad\forall (s,a,s^\prime)\in\cS\times\cA\times\cS,
\end{align*}
where $s_n(h)$ represents the state for arm $n$ at time $h$ and $a_n(h)$ is the corresponding action.  Then R(MA)$^2$B-UCB estimates the true transition probability $\forall (s,a,s^\prime)\in\cS\times\cA\times\cS$ and the true reward $\forall (s,a)\in\cS\times\cA$ by the corresponding empirical averages as
\begin{align*}
    \hat{P}_n(s^\prime|s,a)&=
\frac{T_n(s,a, s^\prime)}{\Lambda(T)},\\
\hat{r}_n(s,a)&=
\frac{1}{\Lambda(T)}\sum_{h=1}^{\Lambda(T)}r_n(s,a;h)\pmb{1}(s_n({h})=s, a_n({ h})=a), \quad\forall (s,a)\in\cS\times\cA. 
\end{align*}
R(MA)$^2$B-UCB further defines confidence intervals for the transition probabilities (resp. the rewards), such that the true transition probabilities (resp. true rewards) lie in them with high probability.   Formally, for $\forall(s,a)\in\cS\times\cA$, we define
\begin{align}\label{eq:confidence_ball}
\cP_n(s,a)&:=\{\tilde{P}_n(s^\prime|s,a), \forall s^\prime\in\cS: |\tilde{P}_n(s^\prime|s,a)-\hat{P}_n(s^\prime|s,a)|\leq {\delta_n(s,a)}\}, \nonumber\\
\cR_n(s,a)&:=\{\tilde{r}_n(s,a): \tilde{r}_n(s,a)=\hat{r}_n(s,a)+\delta_n(s,a)\}, 
\end{align}
where the size of the confidence intervals $\delta_n(s,a)$ is built using the empirical Hoeffding inequality \cite{maurer2009empirical}.  For any $(s,a,s^\prime)\in\cS\times\cA\times\cS$, and $\eta\in(0,1),$ it is defined as 
\begin{align}
    \delta_n(s,a)=\sqrt{\frac{1}{2\Lambda(T)}\log\bigg(\frac{SAN\Lambda(T)}{\eta}\bigg)}.
\end{align}

The set of plausible MDPs associated with the confidence intervals is $\cM=\{M_n=(\cS, \cA, \tilde{P}_n, \tilde{r}_n): \tilde{P}_n(\cdot|s,a)\in\cP_n(s,a), {\tilde{r}_n(s,a)\in\cR_n(s,a)}, \forall n\}$. {Then R(MA)$^2$B-UCB computes a policy by performing optimistic planning.}  Given the set of plausible MDPs, it selects an optimistic transition (resp. reward) function and an optimistic policy by solving a ``modified LP'', which is similar to the LP defined in~(\ref{eq:obj_rel_lp})-(\ref{eq:pmc_rel_lp2}), but with the transition and reward functions replaced by $\tilde{P}(\cdot|\cdot,\cdot)$ and $\tilde{r}(\cdot,\cdot)$ in the confidence balls~(\ref{eq:confidence_ball}) since the corresponding true values are not available. More precisely, R(MA)$^2$B-UCB finds an optimal solution to the following problem
\begin{align}\label{eq:UCBOpt}  
 \underset{{M_n\in\cM}}{\max}& \sum_{t=1}^T\sum_{n=1}^{N}\sum_{(s,a)} \mu_n(s,a;t)\tilde{r}_n(s,a) \nonumber \displaybreak[0]\\
\mbox{ s.t. } & \sum_{n=1}^{N}\sum_{(s,a)} a\mu_n(s,a;t) \leq K,\quad\forall t\in[T],\nonumber\displaybreak[1]\\
& {\sum_{a}} \mu_n(s,a;t)=\sum_{(s^\prime,a^\prime)}\mu_n(s^\prime, a^\prime, t-1)\tilde{P}_n(s|s^\prime, a^\prime),\quad\forall n\in[N],t\in[T],\nonumber\displaybreak[2]\\
& \sum_{a}\mu_n(s,a,1)={\bold{s}_1(s)}, 
\quad\forall s\in\cS.
\end{align}

\noindent\textbf{The extended LP problem.}
The modified LP can be further expressed as an extended LP by leveraging the state-action-state occupancy measure $z_n(s, a, s^\prime, t)$ defined as $z_n(s, a, s^\prime, t)=P_n(s^\prime|s,a)\mu_n(s,a;t)$ to express the confidence intervals of the transition probabilities.  The extended LP over $z=\{z_n\}$ is as follows:
\begin{align}\label{eq:UCB_extended}
\max& \sum_{n=1}^{N}\sum_{t=1}^T\sum_{(s,a,s^\prime)} z_n(s,a, s^\prime;t)\tilde{r}_n(s,a)\nonumber\\
\mbox{ s.t. } & \sum_{n=1}^{N}\sum_{(s,a,s^\prime)} z_n(s,a,s^\prime;t)a\leq K,\quad\forall t\in[T],\nonumber\displaybreak[0]\\
& {\sum_{a, s^\prime}} z_n(s,a, s^\prime;t)=\sum_{s^\prime, a^\prime}z_n(s^\prime, a^\prime, s, t-1),\quad\forall t\in[T], \nonumber\displaybreak[1]\\
& {\sum_{a, s^\prime}} z_n(s,a, s^\prime;1)=\bold{s}_1(s), \quad\forall s\in\cS, \nonumber\displaybreak[2]\\
&\frac{z_n(s,a,s^\prime;t)}{\sum_y z_n(s,a,y;t)}-(\hat{P}_n(s^\prime|s,a)+\delta_n(s,a))\leq0, \quad\forall (s,a,s^\prime,t)\in\cS\times\cA\times\cS\times[T],\nonumber\displaybreak[3]\\
&\hspace{-0.2cm}-\frac{{z_n(s,a,s^\prime;t)}}{\sum_y z_n(s,a,y;t)}+(\hat{P}_n(s^\prime|s,a)-\delta_n(s,a))\leq 0,\quad\forall (s,a,s^\prime,t)\in\cS\times\cA\times\cS\times[T],
\end{align}
where the last two constraints indicate that the transition probabilities lie in the desired confidence interval for $\forall (s,a,s^\prime,t)\in\cS\times\cA\times\cS\times[T]$. Such an approach was also used in \cite{jin2019learning,rosenberg2019online} in the context of adversarial MDPs and \cite{efroni2020exploration,kalagarla2020sample,hasanzadezonuzy2020learning} in constrained MDPs. Once we compute $z$ from~(\ref{eq:UCB_extended}), the policy is recovered from the computed occupancy measures as 
\begin{align}\label{eq:recovered-policy}
\chi_n(s,a;t)=\frac{\sum_{s^\prime}z_n(s,a,s^\prime;t)}{\sum_{b,s^\prime}z_n(s,b,s^\prime;t)}. 
\end{align}
Finally, we compute the  \textit{OMR index} as in~(\ref{eq:occupancy-index}) using~(\ref{eq:recovered-policy}), from which we construct the \textit{OMR Index Policy}, and execute this policy to the end.

\begin{remark}\label{remark-rmabucb}
Although R(MA)$^2$B-UCB looks similar to an ``explore-then-commit'' policy~\cite{ortner2012regret}, a key novelty of R(MA)$^2$B-UCB lies in using the approach of \textit{optimism-in-the-face-of-uncertainty} \cite{jaksch2010near,mete2021reward} to {balance exploration and exploitation in a non-episodic offline manner.}  As a result, there is no need for R(MA)$^2$B-UCB to episodically search for a new MDP instance within the confidence ball with a higher reward as in \cite{ortner2012regret,wang2020restless}, which is computationally expensive (i.e., exponential in the number of arms). The second key novelty is that R(MA)$^2$B-UCB only relies on {samples initially obtained by a generative model to construct a upper-confidence ball}, using which a policy can be derived by solving an extended LP just once, with a computational complexity of $\cO(NSAT)$ (which is $\cO(SAT)$ if all arms are identical).  However, the existing algorithms, e.g. colored UCRL2 are computationally expensive as they rely upon a complex recursive Bellman equation 
in order to derive the policy. Finally, R(MA)$^2$B-UCB uses the structure of our proposed near-optimal index policy in the policy execution phase rather than using a heuristic one as in existing algorithms e.g., \cite{liu2011logarithmic,tekin2011adaptive,tekin2012online,liu2012learning}. These key features ensure that R(MA)$^2$B-UCB achieves almost the same performance as the offline optimum, a sub-linear regret at a low computation expense.
\end{remark}

\subsection{Regret Bound}\label{sec:regret}

We present our main theoretical results in this section.

\begin{theorem}\label{thm:regret}
The regret of the R(MA)$^2$B-UCB policy with $\Lambda(T)=\mathcal{O}(T^{1/2})$ satisfies:
\begin{align}\label{eq:regret}
 \Delta(\pi^\star,\bs_1, T)= \cO\Big({\left(SAK+2K(1+\eta)\right)\sqrt{T}}\Big). 
\end{align}
\end{theorem}
Since there are two phases in R(MA)$^2$B-UCB,  we decompose the regret as $\Delta(\pi^\star,\bs_1,T)=\Delta(T_1)+\Delta(\pi^\star,\bs_1, T_2)$, where $\Delta(T_1)$ is the regret for the planning phase and $\Delta(\pi^\star, \bs_1, T_2)$ is the regret for the policy execution phase with $T_2=T-T_1$. 
The first term $\mathcal{O}(SAK\sqrt{T})$ in \eqref{eq:regret} is the worst regret from $\Lambda(T)$ explorations of each state-action pair under the generative model with $\mathcal{O}(SA\sqrt{T})$ time steps for sampling and at most $K$ arms being activated each time. The second term $\mathcal{O}(2K(1+\eta)\sqrt{T})$ comes from the policy execution  phase.  Specifically, the $\mathcal{O}(2K\eta\sqrt{T})$ regret occurs when $\Lambda(T)$ explorations for each state-action pair construct a set of plausible 
MDPs that do not contain the true MDP $\mathcal{M}$ in line 4 of Algorithm~\ref{alg:UCB}, which is a rare event with probability ${2\eta}/{\Lambda(T)}$.  
The key then is to characterize the regret when the event that the true MDP $\{(\cS, \cA, P_n, {r}_n), \forall n\}$ lies in the set of plausible MDP $\cM$ occurs. 
{Based on the optimism of plausible MDPs, the optimal average reward $\tilde{\xi}$ for the optimistic MDP $\{(\cS, \cA, \tilde{P}_n, \tilde{r}_n), \forall n\}$ is no less than $\xi^{opt}$.} Therefore the expected regret can be bounded by $T_2\tilde{\xi}-T_2\xi^{opt},$ which is directly related with the occupancy measure we defined.

\begin{remark}
Though R(MA)$^2$B-UCB is an offline non-episodic algorithm,  {it still achieves an $\tilde{\cO}(\sqrt{T})$ regret no worse than the episodic colored-UCRL2.}  Note that for colored-UCRL2, the regret bound is instance-dependent due to the online episodic manner such that the regret bound tends to be logarithmic in the horizon as well.  However, R(MA)$^2$B-UCB adopts explore-then-commit mechanism which uses generative model based sampling and constructs the plausible MDPs sets only once. This removes the instance-dependent regret with order of $\log T$. Though the state-of-the-art Restless-UCB \cite{wang2020restless} has a similar mechanism as ours in obtaining samples in an offline manner, it lowers its implementation complexity by sacrificing the regret performance to $\mathcal{O}(T^{2/3})$ since it heavily depends on the performance of an offline oracle approximator for policy execution. Instead, we leverage our proposed provably optimal and computationally appealing index policy for the policy execution phase. This also contributes to the low multiplicative ``pre-factor'' in the regret. 
\end{remark}

\section{Experiments}\label{sec:exp}
We now present our experimental results that validate our model and theoretical results.  These verify the asymptotic optimality of the \textit{OMR Index Policy}, 
and the sub-linear regret of the R(MA)$^2$B-UCB policy. 
In particular, we evaluate the R(MA)$^2$B-UCB policy under two real-world applications of restless bandit problems, namely ``a deadline scheduling problem'' where each arm has binary actions, and ``dynamic video streaming over fading wireless channel'' where each arm has multiple actions, using real video traces.

\subsection{Evaluation of the \textit{OMR Index Policy}}\label{sec:sub1}
\textbf{\textit{Binary actions:}}  
Since most existing index policies are designed only for the conventional binary action settings in which arms are chosen to be either active or passive, and cannot be applied to the multi-action setting that is considered in our paper, we first consider a controlled Markov process in which there are two actions for each arm, and the states evolve as a specific birth-and-death process where state $s$ can only transit to $s-1$ or $s+1$.  We compare \textit{OMR Index Policy} with the following popular state-of-the-art index policies: Whittle index policy \cite{whittle1988restless}, the { Fluid-priority policy of~\cite{zhang2021restless}}, and a priority based policy proposed in~\cite{verloop2016asymptotically}.  We consider a setup with 10 classes of arms, in which each arm has a state space $\mathcal{S}=\{1,2,3,4,5,6,7,8,9,10\}$.  The arrival rates $\lambda$ are set as $\{3, 6, 9, 12, 15, 18, 21, 24, 27, 30\}$ with a departure rate $\mu=20$.  The controlled transition probabilities satisfy  $P(s,s+1)=\lambda/(\lambda+\mu)$ and  $P(s,s-1)=\mu/(\lambda+\mu)$.  When a class-$i$ arm is activated, it receives a random reward $r_i(s)$ that is a Bernoulli random variable with a state dependent rate $s\cdot p_i$, i.e., $r_i(s)\sim Ber(sp_i)$ where $p_i$ uniformly distributed in $[0.01, 0.1]$. If the arm is not activated then no reward is received.  The time horizon is set to { $T=100$} and the activation ratio is set to $\alpha=K/N=0.3$.  For the ease of exposition, we let the number of arms vary from $50$ to $400.$   

\begin{figure}
	\center
	\begin{minipage}[b]{.45\textwidth}
    	\includegraphics[width=0.99\textwidth]{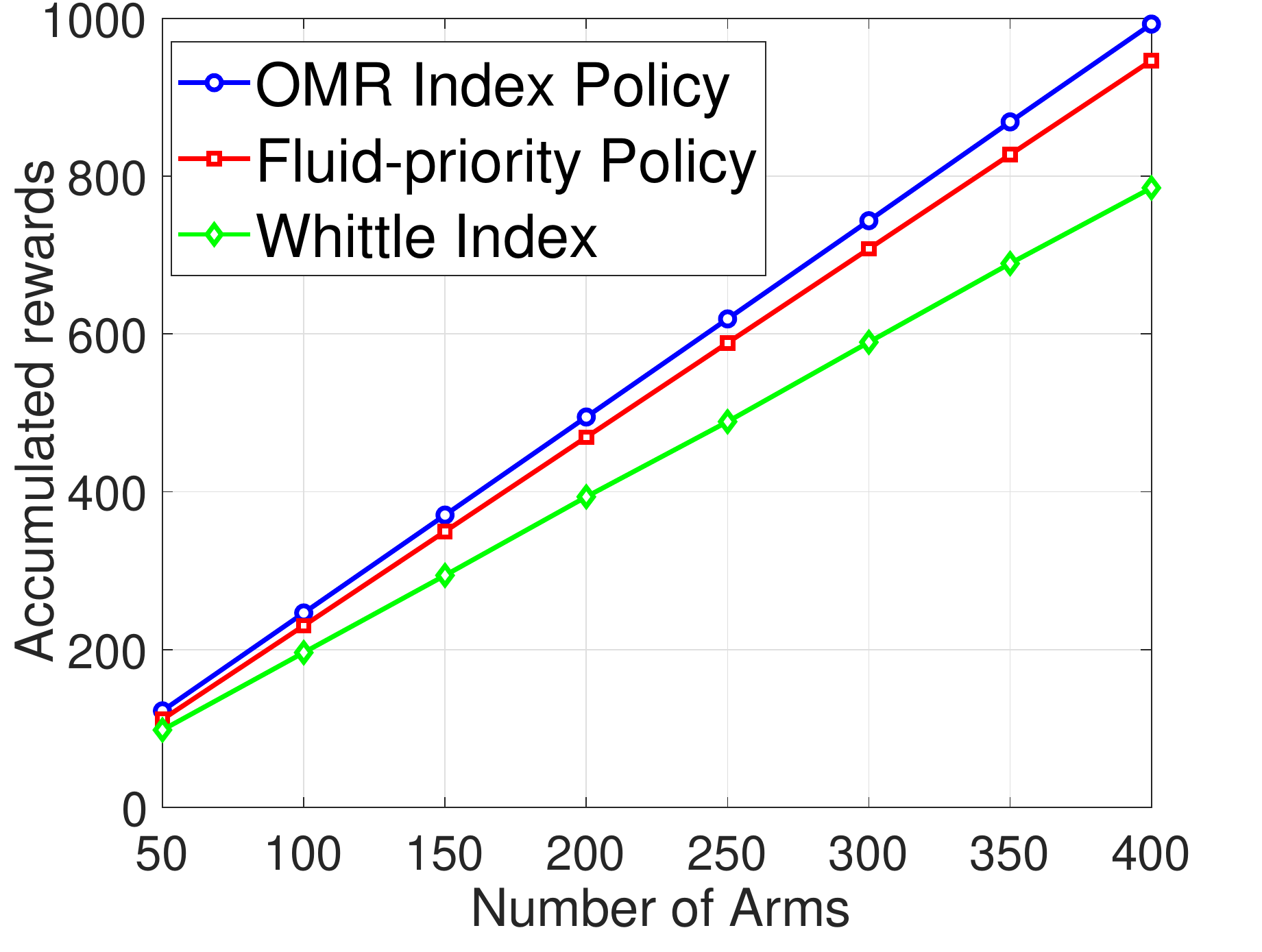}
    	\vspace{-0.25in}
		\caption{{Accumulated reward: binary action setting.}}
	\label{fig:Accumulated_reward}
	\end{minipage}
	\begin{minipage}[b]{.45\textwidth}
	    \includegraphics[width=0.99\textwidth]{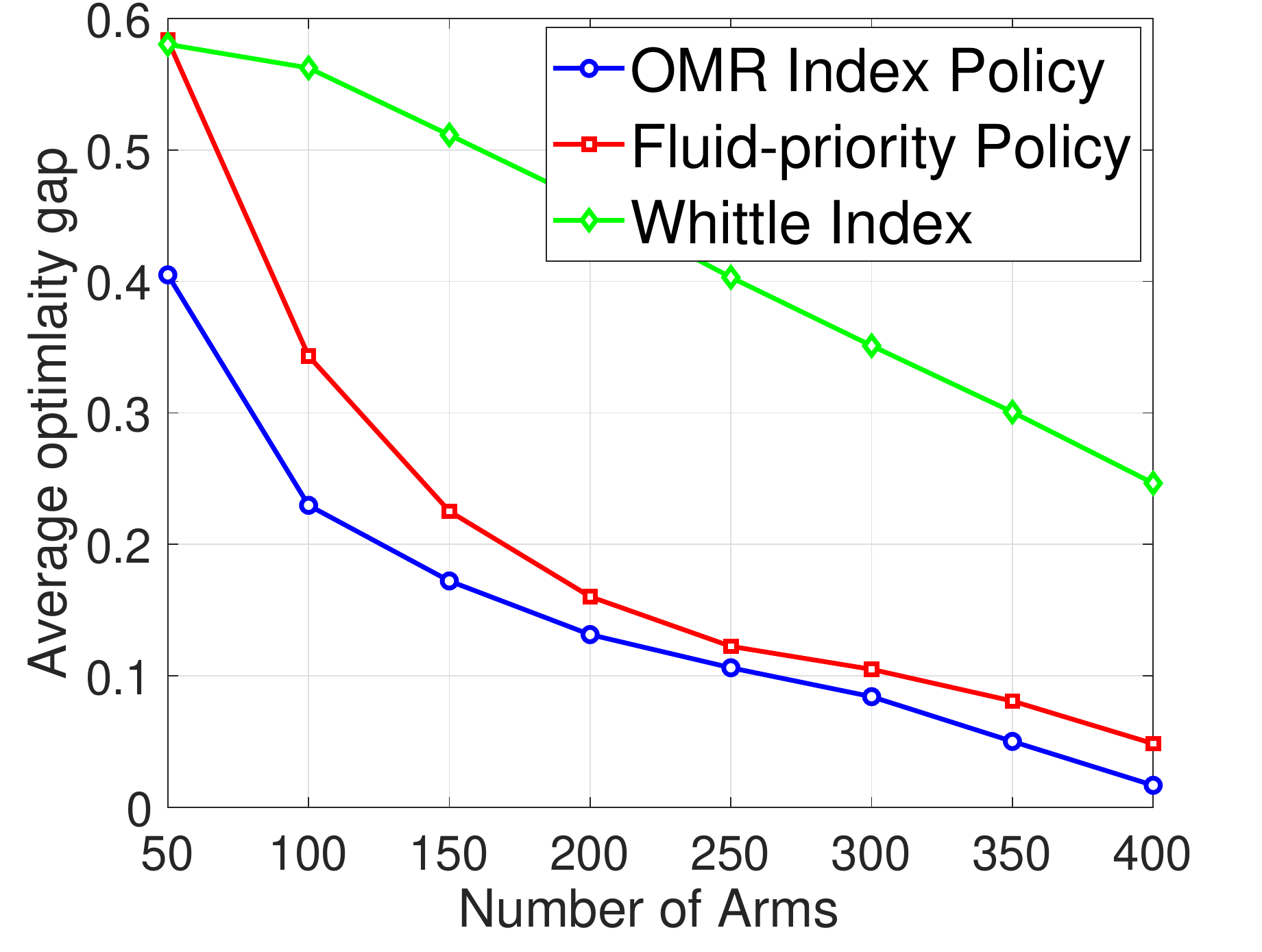}
	    \vspace{-0.25in}
		\caption{{Average optimality gap: binary action setting.}}
		\label{fig:Opt-gap}
	\end{minipage}
	  \vspace{-0.1in}
\end{figure}

The cumulative rewards collected by these policies are presented in Figure~\ref{fig:Accumulated_reward}.  We observe that \textit{OMR Index Policy} performs slightly better than the Fluid-priority policy.  { We conjecture that this is due to the fact that \textit{OMR Index Policy} prioritizes the arms directly based on their contributions to the cumulative reward, while Fluid-priority policy does not differentiate arms in the same priority category.}  More importantly, both \textit{OMR Index Policy} and  Fluid-priority policy
significantly outperform the Whittle index policy.

We further validate the asymptotic optimality of \textit{OMR Index Policy} (see Theorem~\ref{thm:asym_opt}).  In particular, we compare the rewards obtained by \textit{OMR Index Policy} and the two baselines, with that obtained from the theoretical upper bound obtained by solving the LP in \eqref{eq:obj_rel_lp}-\eqref{eq:pmc_rel_lp2}. We call this difference as the optimality gap. 
The average optimality gap, i.e., the ratio between the optimality gap and the number of arms of different policies is illustrated in Figure~\ref{fig:Opt-gap}. Again, we observe that \textit{OMR Index Policy} slightly outperforms the Fluid-priority in terms of the vanishing speed of the average optimality gap since \textit{OMR Index Policy} achieves a higher accumulated reward as shown in Figure~\ref{fig:Accumulated_reward}. Moreover, both \textit{OMR Index Policy} and Fluid-priority significantly outperform the Whittle index policy.  This is due to the fact that the optimality gap of the Fluid-priority index policy (i.e. a constant $\mathcal{O}(1)$) does not scale with the number of arms, while that of Whittle index policy does \cite{zhang2021restless}.

\begin{wrapfigure}{rt}{0.5\linewidth}
    \centering
    \vspace{-0.2in}
       \includegraphics[width=0.5\textwidth]{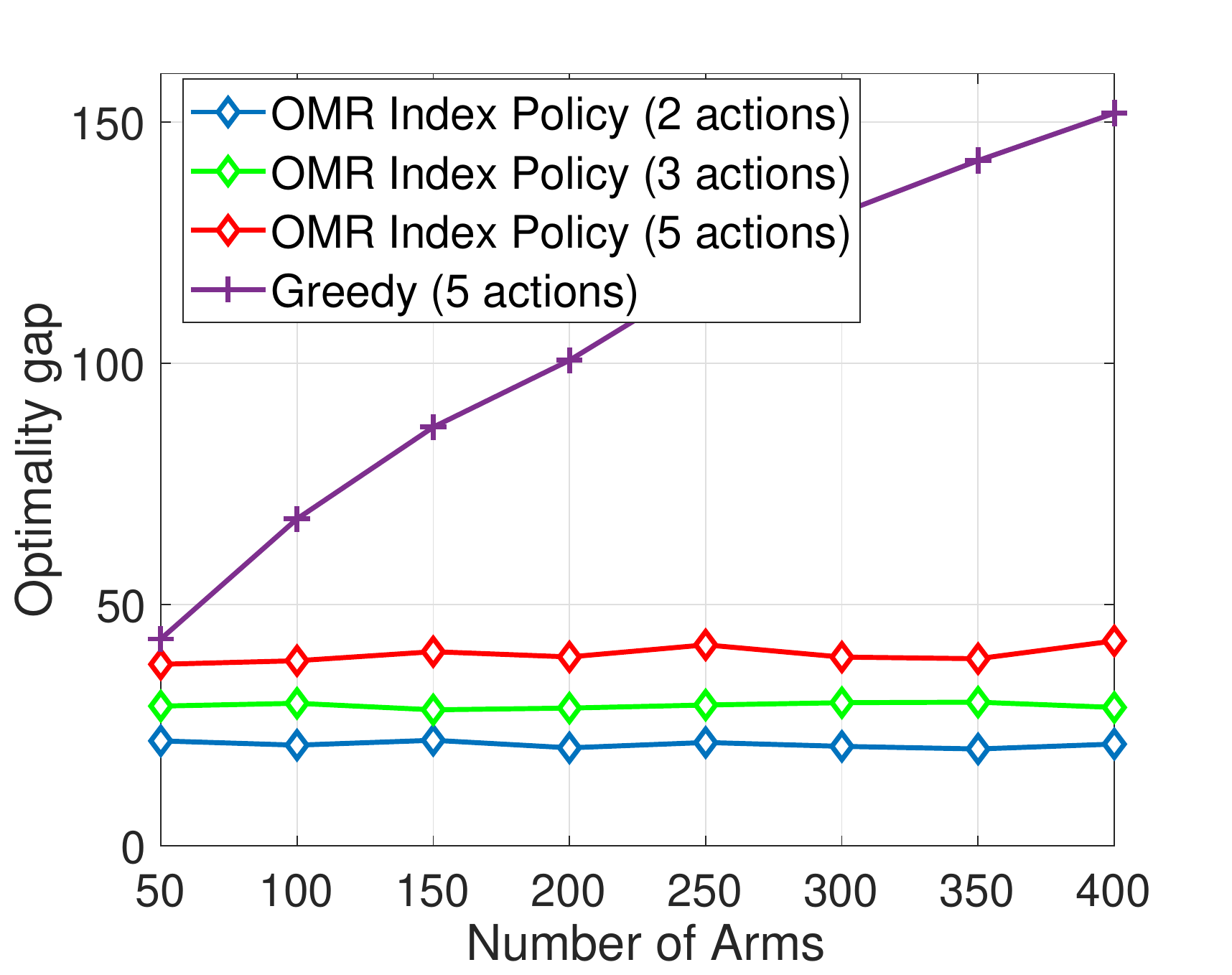}
    	\vspace{-0.1in}
		\caption{{Optimality gap: multi-action setting.}}
	\label{fig:OptimalityGap-multi}
\end{wrapfigure}

\textbf{\textit{Multiple actions:}} We further evaluate our index policy for the general multi-action setup, and consider a more general Markov process in which any two arbitrary states could communicate with each other. The controlled transition probabilities are generated randomly. For the ease of exposition, we consider the number of actions for each arm to assume values from the set $\{2,3,5\}$.  Our results and observations hold for other numbers of arms.  Note that most existing index policies including the two state-of-the-art index policies considered above are designed only for the conventional binary action setup and cannot be applied to the multi-action setting considered in this paper.  To this end, we compare \textit{OMR Index Policy} with the ``greedy policy,'' that at each time selects actions that yield the maximum reward. Note that the choice of action would depend upon the current states, since the rewards depend upon state values. The performance in terms of optimality gap is shown in Figure \ref{fig:OptimalityGap-multi}.~Firstly, we observe that the optimality gap slightly increases as the number of available actions increases while the number of arms is kept fixed. The impact of such marginal increase vanishes as the number of arms increases.  Similar to the observations made in the case of binary actions, this indicates the asymptotic optimality of our proposed \textit{OMR Index Policy}. Secondly, \textit{OMR Index Policy}  significantly outperforms the greedy policy, whose optimality gap increases with the number of arms.

\begin{figure*}
	\center
	\begin{minipage}[b]{.45\textwidth}
    	\includegraphics[width=0.99\textwidth]{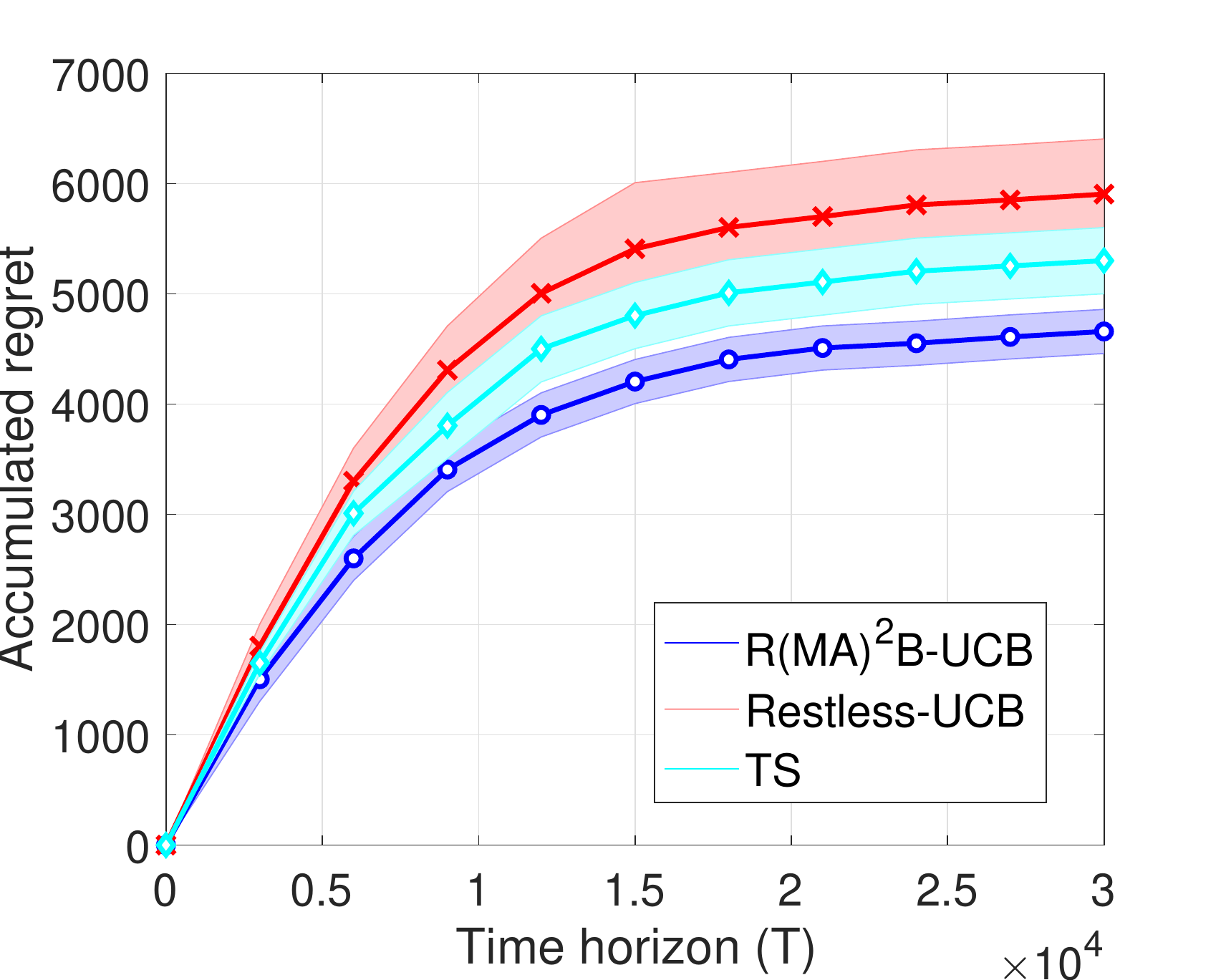}
    	\vspace{-0.25in}
		\caption{Comparison of accumulated regret: binary action setting.}
	\label{fig:regret}
		\end{minipage}
	\begin{minipage}[b]{.45\textwidth}
    	\includegraphics[width=0.99\textwidth]{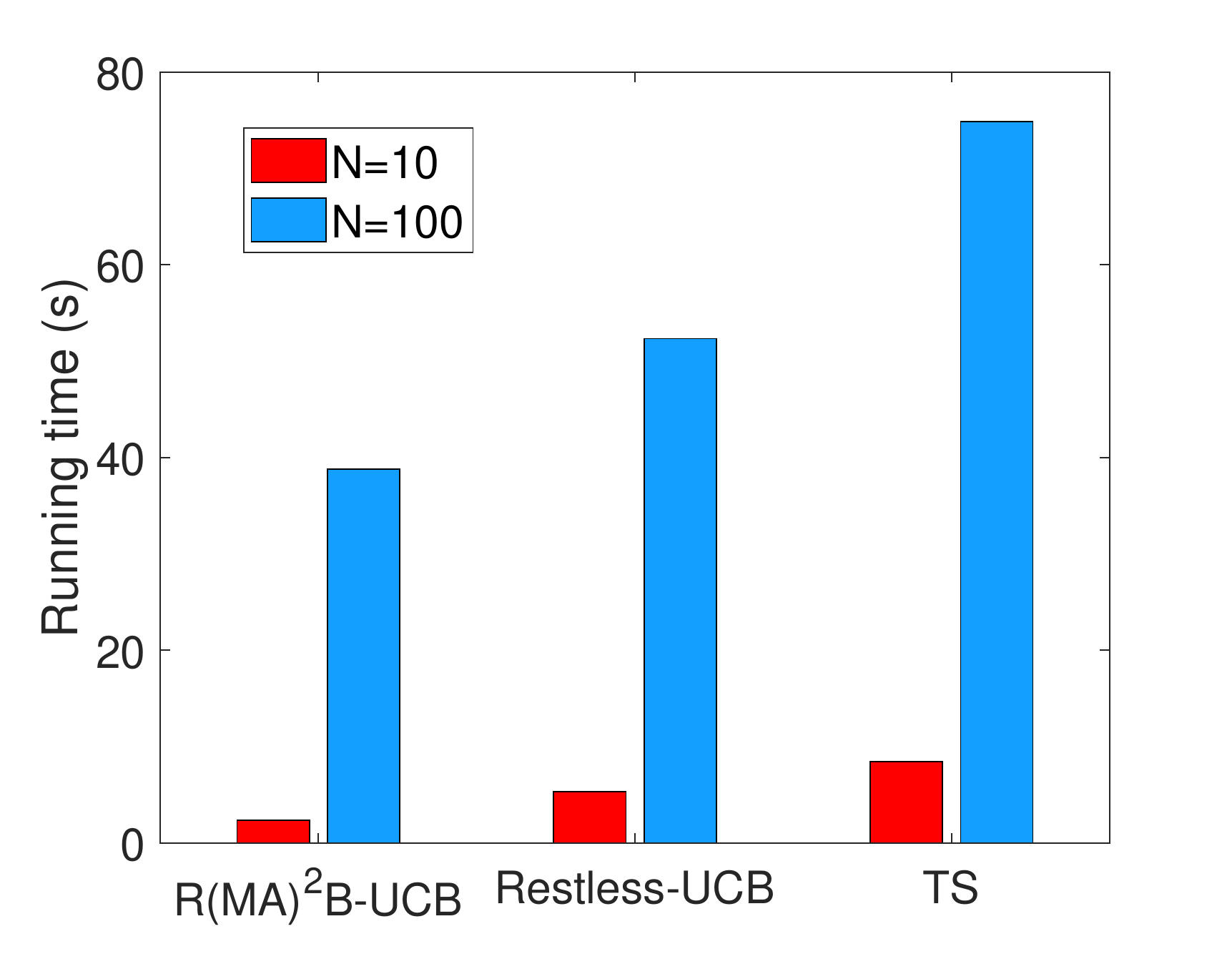}
    	\vspace{-0.25in}
		\caption{Comparison of average running time: binary action setting.}
	\label{fig:running_time}
	\end{minipage}
	\vspace{-0.15in}
\end{figure*}

\subsection{Evaluation of the R(MA)$^2$B-UCB Policy}

\textbf{\textit{Binary actions:}} We then evaluate the performance of  R(MA)$^2$B-UCB.  We compare with two state-of-the-art algorithms including Restless-UCB \cite{wang2020restless} and a Thompson sampling (TS) based policy \cite{jung2019regret} for restless bandits.  Note that Restless-UCB is also an offline learning policy similar to ours while the TS-based policy is an online policy that has a sub-linear regret in the Bayesian setup but suffers from a high computation complexity. Colored-UCRL2~\cite{ortner2012regret} is a popular algorithm for RMAB problems, but it is well known that the computational complexity of colored-UCRL2 grows exponentially with the number of arms. Furthermore, it has been shown in~\cite{wang2020restless} that Restless-UCB outperforms colored-UCRL2, and hence we do not include it in our experiments.

We use the same settings for experiments as was described above for evaluating our index policy  and for simplicity {choose the number of arms $N$ to be $100$,} though the results for a larger number of arms would be similar.  For the TS-based policy, we set the prior
distribution to be uniform over a finite support $\{0, 0.1, 0.2, \ldots, 0.9, 1.0\}.$  
Regrets of these algorithms are shown in Figure~\ref{fig:regret}, in which we use the Monte Carlo simulation with $1,000$ independent trials. 
R(MA)$^2$B-UCB  achieves the lowest cumulative regret. 
An explanation behind this phenomenon is that Restless-UCB sacrifices the regret performance for a lower computational complexity, and hence performs worse as compared with the online TS-based policy. R(MA)$^2$B-UCB achieves the best performance, which can be partly explained by the near-optimality of our index policy (see Remark~\ref{remark-rmabucb}). When the number of samples are sufficiently large, i.e $T$ is large), R(MA)$^2$B-UCB achieves a near optimal performance.

\begin{wrapfigure}{rt}{0.5\linewidth}
    \centering
       \includegraphics[width=0.5\textwidth]{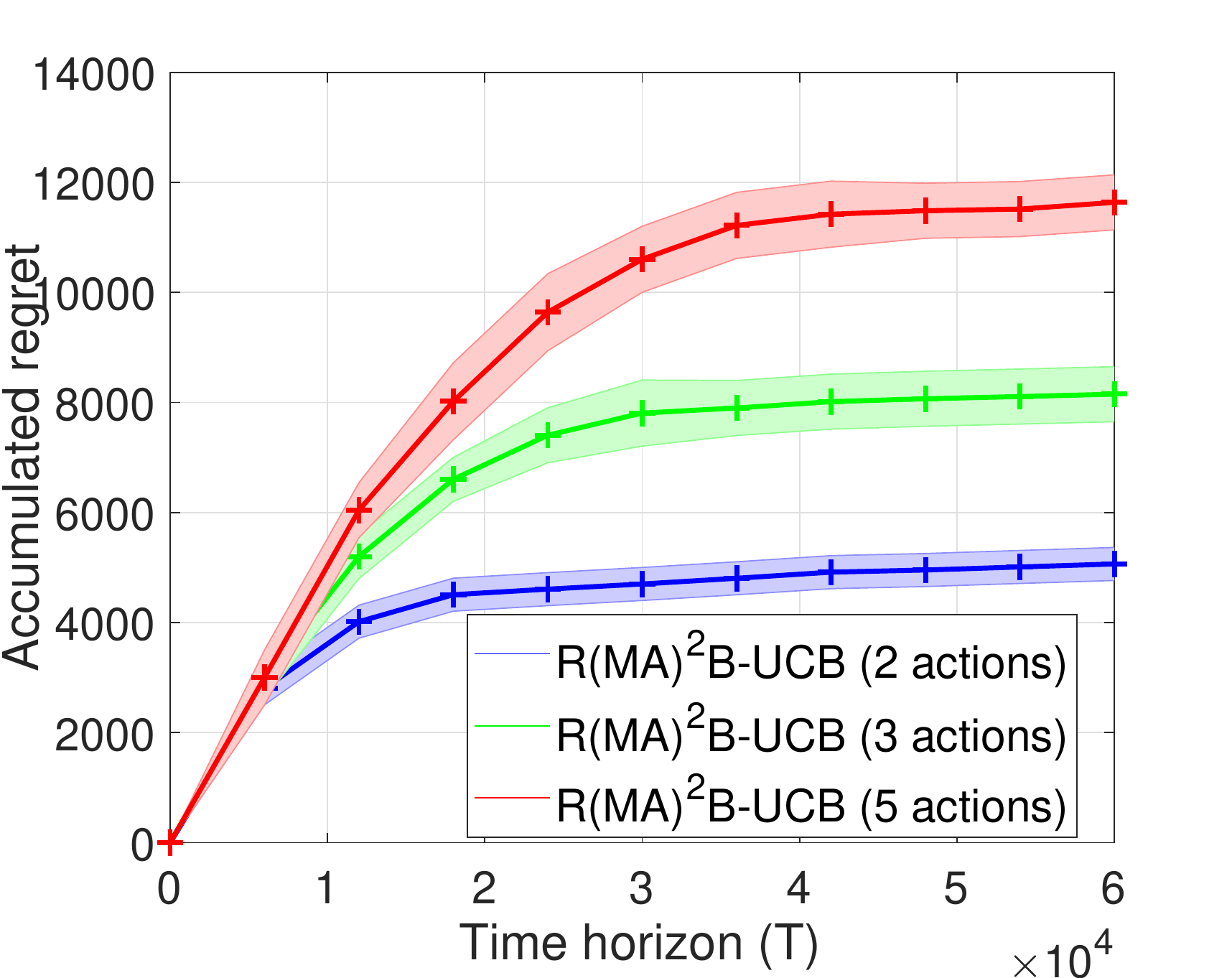}
    	\vspace{-0.2in}
		\caption{Comparison of accumulated regret: multi-action setting.}
	\label{fig:regret_over_time}
\end{wrapfigure}

We also compare the average run time of different algorithms. In this experiment, the horizon is $T=60,000$. The results are presented in Figure  \ref{fig:running_time}, which are averaged over $100$ Monte Carlo runs of  a single-threaded program on Intel Core i5-6400 desktop with 16 GB RAM. It is clear that R(MA)$^2$B-UCB is more efficient in terms of run time.  For example, R(MA)$^2$B-UCB reduces the run time by up to $52\%$ (resp. $70\%$) as compared with the Restless-UCB (resp. TS-based policy) when there are $10$ arms, and reduces the corresponding run time by up to $26\%$ (resp. $48\%$) when there are $100$ arms.  The improvement over the colored-UCRL2 is even more significant when the number of arms is larger, since the time complexity of colored-UCRL2 grows exponentially with the number of arms.  Hence we omit the comparison here. 
A significant improvement comes from the intrinsic design of our policy which only needs to solve an LP once, while {the Restless-UCB needs a computation-intensive numerical approximation of the Oracle (e.g., Whittle index policy)} and the TS-based policy is an online episodic algorithm that solves a Bellman equation for every episode.

\textbf{\textit{Multiple actions:}} {We further evaluate R(MA)$^2$B-UCB under multi-action settings by considering a more general Markov process in which any two arbitrary states may communicate with each other and the transition probability matrices are randomly generated. The other settings remain the same as in the index policy  evaluation. For the ease of exposition, we consider the number of actions to be $2,3$ and $5$.  Figure \ref{fig:regret_over_time} shows the accumulated regret vs. time for R(MA)$^2$B-UCB under different numbers of actions.  Since the Restless-UCB and TS-based policies are hard to be extended to the multi-action setting, we do not consider them in this comparison.  From Figure \ref{fig:regret_over_time}, we observe that R(MA)$^2$B-UCB achieves $\sqrt{T}$ regret under multi-action settings, which validates our theoretical contributions in the paper (see Theorem~\ref{thm:regret}).  Furthermore, when the number of actions increases, it takes a larger number of time steps for the accumulated regret to converge.  
In other words, the planning phase in R(MA)$^2$B-UCB (see Algorithm~\ref{alg:UCB}) will take a longer time to learn the system parameters.  }

\textbf{\textit{Case Study: A Deadline Scheduling Problem.}}  
We consider the  deadline scheduling problem \cite{yu2018deadline} for the scheduling of electrical
vehicle charging stations.  A charging station (agent) has total $N$ charging spots (arms) and can charge $M$ vehicles in each round. The charging station
obtains a reward for each unit of electricity that it provides to a vehicle and receives a penalty (negative reward) when a vehicle is not fully charged. The
goal of the station is to maximize its net reward. We use exactly the same
setting as in \cite{yu2018deadline} for our experiment. 
More specifically, the state of an arm is denoted by a pair of integers $(D; B)$, where $B$ is the amount of
electricity that the vehicle still needs and $D$ is the time until the vehicle leaves the station.  When a
charging spot is available, its state is $(0; 0)$. $B$ and $D$ are upper-bounded by $9$ and
$12$, respectively.  Hence, the size of state space is $109$ for each arm.
The reward received by agent from arm $i$ is as follows,
\begin{align*}
    r_i(t)=\begin{cases}
    (1-0.5)a_i(t), \quad \text{if $B_i(t)>0, D_i(t)>1$},\\
    (1-0.5)a_i(t)-0.2(B_i(t)-a_i(t))^2, \text{if $ B_i(t)>0, D_i(t)=1$},\\
    0, \quad \text{Otherwise},
    \end{cases}
\end{align*}
where $a_i(t)=0$ means being passive and $a_i(t)=1$ means being active.  The state transition satisfies 
\begin{align}\nonumber
    S_i(t+1)=\begin{cases}
    (D_i(t)-1, B_i(t)-a_i(t)), \quad\text{if $D_i(t)>1$},\\
    (D, B), \quad\text{with prob. 0.7 if $D_i(t)\leq 1$},
    \end{cases}
\end{align}
where $(D, B)$ is a random state when a new vehicle arrives at the charging spot $i$. There are total $N=100$ charging spots and a maximum $M=30$ can be served at each time. 

\begin{wrapfigure}{rt}{0.5\linewidth}
    \centering
       \includegraphics[width=0.5\textwidth]{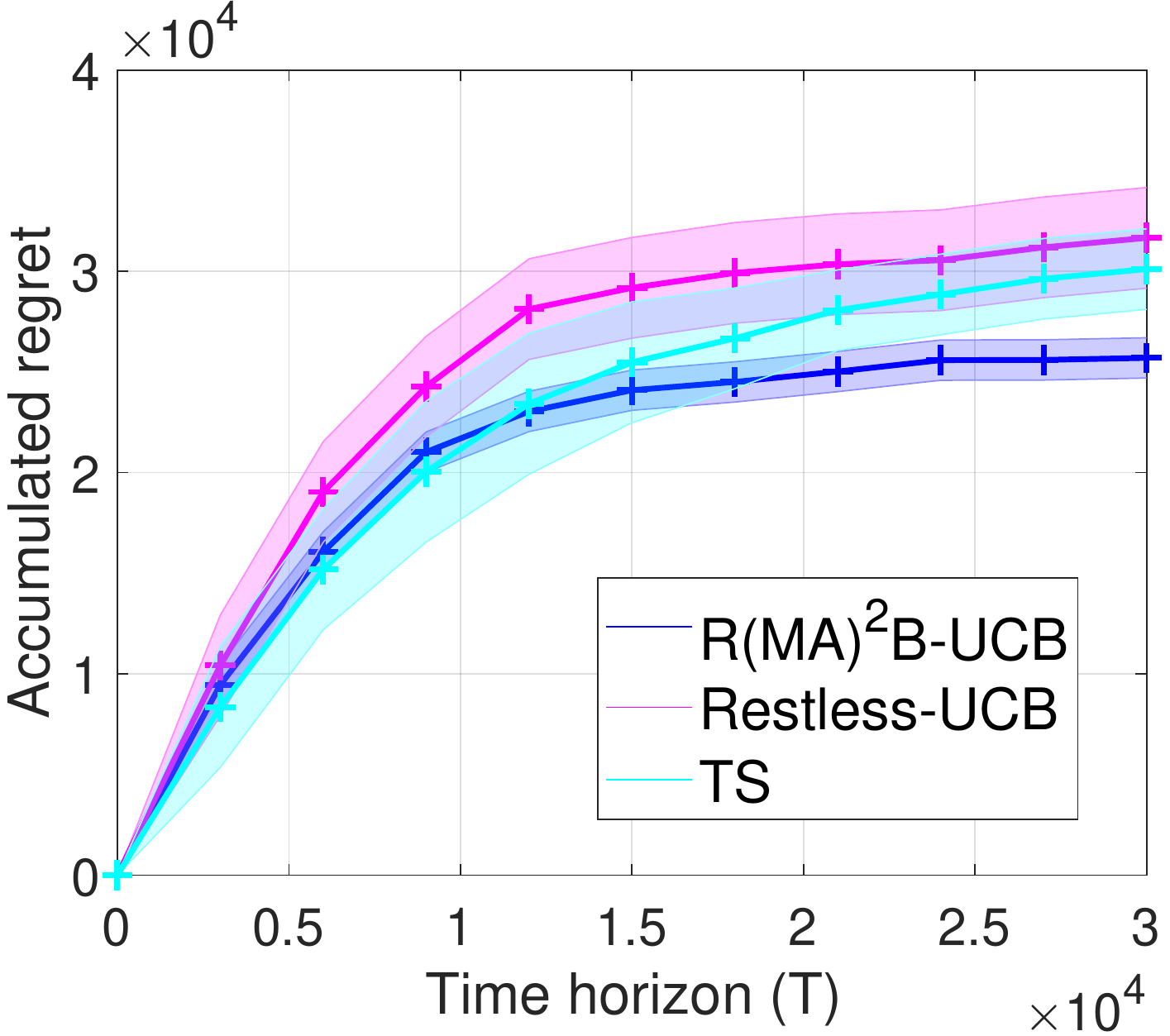}
    	\vspace{-0.2in}
		\caption{Comparison of accumulated regret in the deadline scheduling problem.}
	\label{fig:regret_case_study}
\end{wrapfigure}

We compare the learning performance of our R(MA)$^2$B-UCB with the two state-of-the-art algorithms, i.e., Restless-UCB and a Thompson sampling (TS)-based policy for this deadline scheduling problem, which is shown in Figure \ref{fig:regret_case_study}.  We observe that all three polices achieve sub-linear regrets, which is consistent with their theoretical performance.  Our R(MA)$^2$B-UCB performs better than the other two state-of-the-art algorithms.~Note that the TS-based policy has a lower cumulative regret when the number of time steps is small as compared with the other two policies.~This is because the TS-based policy is an episodic algorithm that improves the policy episode-by-episode while the R(MA)$^2$B-UCB and Restless-UCB run according to a fully random policy at the exploration phase.

\textbf{\textit{Case Study: A Dynamic Video Streaming Over Fading Wireless Channel.}} 
We consider the adaptive video streaming problem, where multiple users compete for network resources in order to deliver video packets over unreliable wireless channels.~This problem can be cast as a restless multi-armed bandit problem that has multiple actions \cite{singh2015optimizing}. In particular, an access point connected to a server dynamically controls (i) the quality of video chunks that are delivered, and (ii) the channel resources, e.g. transmission power that are allocated to $N$ users. These decisions are dynamic, i.e. based on each user's current state, which in this case turns out to be the same as the remaining playtime.  The goal is to maximize the total expected quality of experience  (QoE).

The state of user $n$ at time $t$ is defined as $S_n(t):=(B_n(t), \Gamma_n(t))$ with $B_n(t)$ being the remaining play time of video chunks in its buffer, and $\Gamma_n(t)$  being the quality of last successfully received video chunk before time $t$.  Specifically, $B_n(t)=0$ represents the occurrence of a rebuffering event. The action for user $n$ at time $t$ determined by the access point is denoted as $A_n(t):=(R, W)$, where $R\in\mathcal{R}$ is the quality of the video chunk, and $W\in\mathcal{W}$ is the network resources allocated.  The buffer length remains the same when one chunk is successfully transmitted to the user, otherwise the buffer length decreases by $L$ seconds. Therefore, the transition probability of the MDP associated with user $n$ is expressed as follows,
\begin{align*}
 \mathbb{P}\Big(S_n(t+1)&=((B-L)_+,\Gamma)|S_n(t)=(B,\Gamma), A_n(t)=(R,0)\Big)=1,
\end{align*}
when a passive action is selected.  When action $(R,W)$ is chosen for transmitting a video chunk to user $n$ at $t$, we have
\begin{align*}
&\mathbb{P}\Big(S_n(t+1)=(B, R)|S_n(t)=(B,\Gamma),A_n(t)=(R,W)\Big)=\mathbb{P}(R,W), \displaybreak[1]\\
 &\mathbb{P}\Big(S_n(t+1)=((B-L)_+,\Gamma)|S_n(t)=(B,\Gamma), A_n(t)=(R,W)\Big)= 1-\mathbb{P}(R,W).
\end{align*}
Note that if $B=0$, user $n$ suffers from a rebuffering event, and the length increases by $L$ seconds if one chunk is successfully transmitted to user $n$. The instantaneous reward received by user $n$ at time $t$ is defined by the QoE function as follows,
\begin{align*}
    \text{QoE}_n(t)=R \mathbb{P}(R,W)-{1}_{\{B=0\}}-\left|R-\Gamma\right| \mathbb{P}(R,W).
\end{align*}

We evaluate the performance of {R(MA)$^2$B-UCB Policy} for adaptive video streaming over wireless networks using real video traces \cite{lederer2012dynamic}.  All videos are encoded into multiple chunks, with each chunk having a playtime of one second.  Each video consists of three resolutions: $360$p, $720$p and $1080$p, from which we abstract the bitrate levels as $\mathcal{R}=\{1,2,3\}$. We consider $N=10$ users, and the total network resource is $15$ Mbps with $\mathcal{W}=\{0, 1 \text{Mbps}, 3 \text{Mbps}, 5 \text{Mbps}\}$. Denote each resource level in $\mathcal{W}$ as index $0, 1,2$ and $3$, respectively. Therefore, in total there are $10$ different actions and the successful transmission probabilities under this trace are then calculated as follows,
$\mathbb{P}(0,0)=0$, $\mathbb{P}(1, 1)=1$, $\mathbb{P}(2,1)=0.293$, $\mathbb{P}(3,1)=0.01$, $\mathbb{P}(1,2)=1$, $\mathbb{P}(2,2)=0.57$, $\mathbb{P}(3,2)=0.01$, $\mathbb{P}(1,3)=1$, $\mathbb{P}(2,3)=1$, $\mathbb{P}(3,3)=0.6.$

\begin{figure}
	\center
	\begin{minipage}[b]{.45\textwidth}
    	\includegraphics[width=0.99\textwidth]{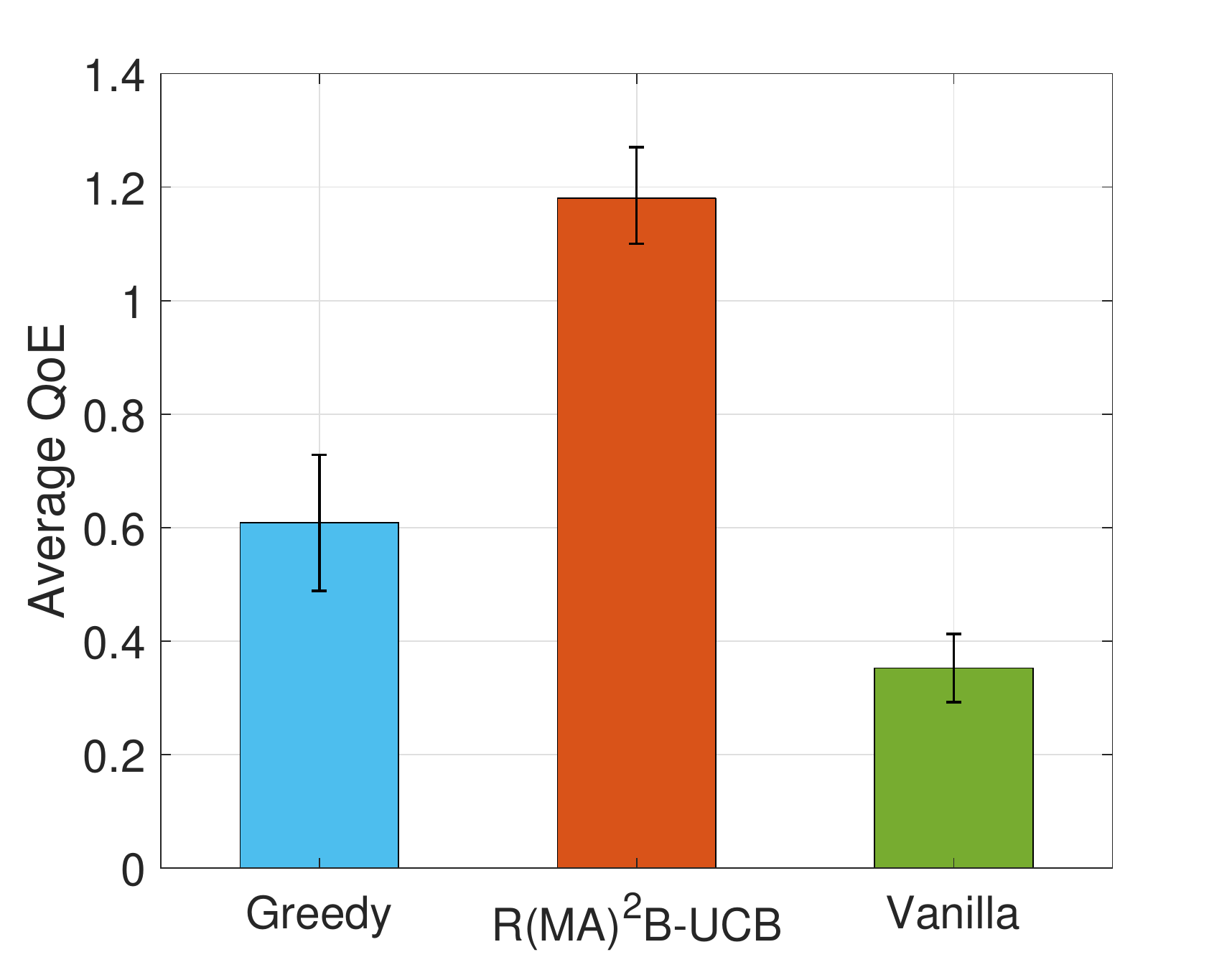}
    	\vspace{-0.25in}
		\caption{Comparison of average QoE in the wireless video streaming problem.}
	\label{fig:QoE_case_study}
	\end{minipage}
	\begin{minipage}[b]{.45\textwidth}
	    \includegraphics[width=0.99\textwidth]{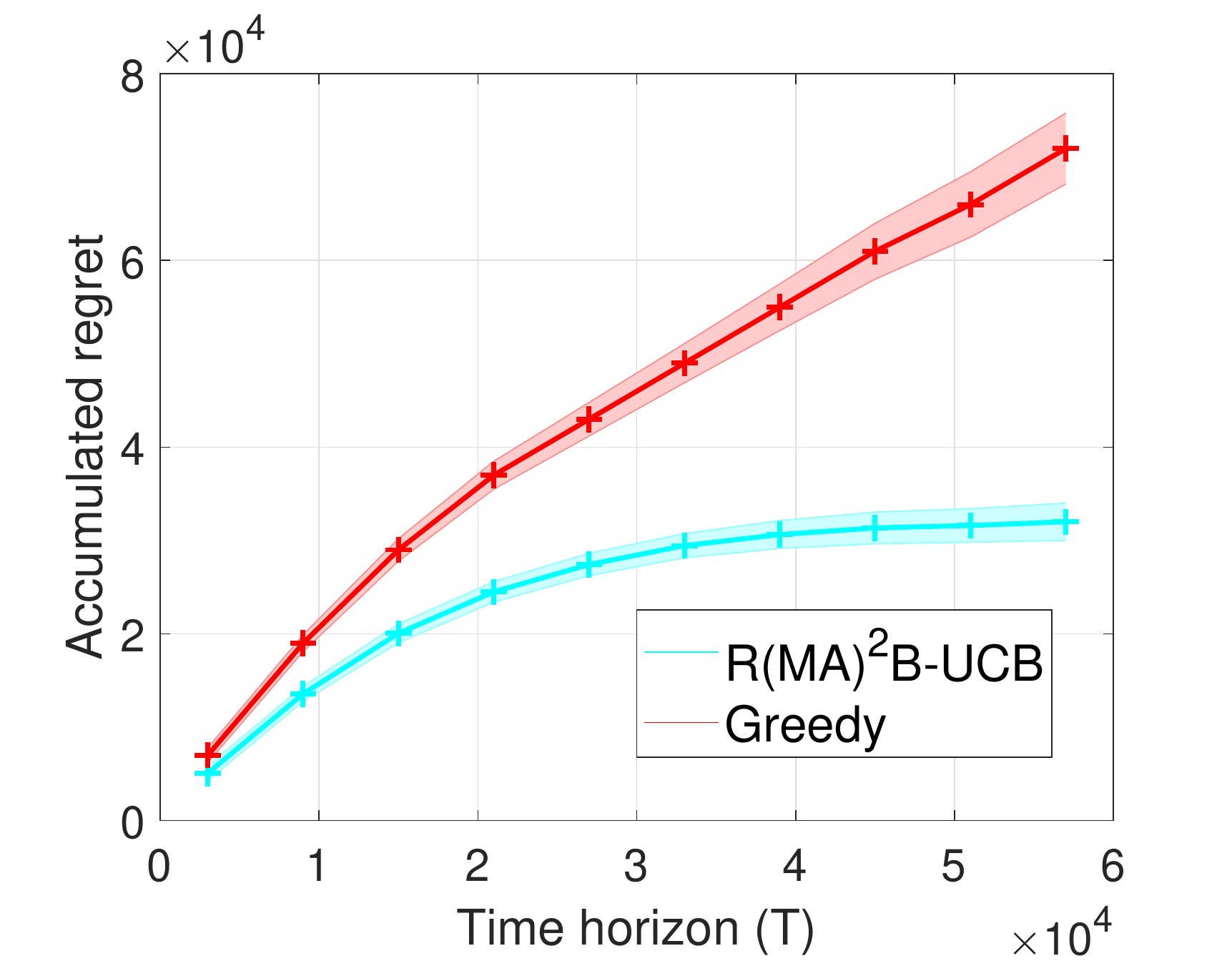}
	    \vspace{-0.25in}
		\caption{Comparison of accumulated regret in the wireless video streaming problem.}
	\label{fig:regret2_case_study}
	\end{minipage}
	  \vspace{-0.1in}
\end{figure}

Since the Restless-UCB and TS-based policies cannot be directly extended to multi-action settings, we compare the learning performance of our R(MA)$^2$B-UCB with the two well-known heuristic algorithms, i.e., Greedy and Vanilla.  In particular, \textit{Vanilla} is a base case with served users being allocated
the highest resources, and no differentiation between users, and \textit{Greedy} is the case where each user greedily selects the action with the largest reward for current state.  The average QoE achieved by these policies are shown in Figure \ref{fig:QoE_case_study}.  We observe that R(MA)$^2$B-UCB significantly outperforms the two heuristic algorithms with the highest average QoE.  We further evaluate the corresponding learning regret as shown in Figure \ref{fig:regret2_case_study}. Since Greedy significantly outperforms Vanilla in average QoE and hence we do not include Vanilla in this comparison.  It is clear that R(MA)$^2$B-UCB achieves a $\sqrt{T}$ regret while Greedy achieves nearly linear regret as $T$ grows large.

\section{Conclusion}\label{sec:conclusion}

In this paper, we studied an important extension of the popular restless multi-armed bandit problem that allows for choosing from multiple actions for each arm, which we denote by R(MA)$^2$B. We firstly proposed a computationally feasible index policy dubbed \textit{OMR Index Policy}, and showed that it is asymptotically optimal.  Since the system parameters are often unavailable in practice, we then developed a learning algorithm that learns the index policy.  We combine a generative approach to reinforcement learning with the UCB strategy to get the R(MA)$^2$B-UCB algorithm. It enjoys a low learning regret since it can fully exploit the structure of the proposed \textit{OMR Index Policy}.  We also show that R(MA)$^2$B-UCB achieves a sub-linear regret.  Our experimental results further showed that R(MA)$^2$B-UCB outperforms other state-of-the-art algorithms.

\bibliographystyle{unsrt}  
\bibliography{refs}

\clearpage
\appendix

\section{Proof of Theorem~\ref{thm:asym_opt}}

To prove Theorem~\ref{thm:asym_opt}, we first introduce some auxiliary notations. Let $B_n(s; t)$ be the number of class-$n$ arms at state $s$ at time  $t$ and $D_n(s,a;t)$ be the number of class-$n$ arms at state $s$ at time $t$ that are being activated by action $a\in\mathcal{A}\setminus\{0\}$.  In the following, we show that when the number of each class of arms $\rho$ goes to infinity, the ratios $B_n(s; t)/\rho$ and $D_n(s,a; t)/\rho$ converge.

\begin{lemma}\label{lemma1}
For $\forall n\in[N]$ and $\forall t\in[T]$, we have 
\begin{align*}
\lim_{\rho\rightarrow \infty} \frac{B_n(s;t)}{\rho}&=P_n(s;t);\\
\lim_{\rho\rightarrow \infty} \frac{D_n(s,a;t)}{\rho}&=P_n(s;t)\chi_n^\star(s, a;t), ~\forall a\in\mathcal{A}\setminus\{0\}.
\end{align*}
\end{lemma}

\begin{proof}
We prove the above equations by induction.  
When $t=1$, denote the initial state of each arm $n$ as $s_n(1)$, and we have
\begin{align*}
 \lim_{\rho\rightarrow\infty} \frac{B_n(s_n(1);1)}{\rho}&=\lim_{\rho\rightarrow\infty} \frac{\rho}{\rho}=1
 =P_n(s_n(1);1), ~\forall n\in[N].   
\end{align*}

Meanwhile, denote $D_n(s_n(1),a;1)=\chi_n^\star(s_n(1),a;1)\rho$,  and we have
\begin{align*}
\lim_{\rho\rightarrow\infty} \frac{D_n(s_n(1),a;1)}{\rho}&=\chi_n^\star(s_n(1),a;1)
=P_n(s_n(1);1)\chi_n^\star(s_n(1),a;1).
\end{align*}

Now we assume that the equations hold at time $t$. Then we need to show that these conditions also hold for time $t+1.$  

We first show that this is true for the first equation in Lemma~\ref{lemma1}. Denote $C_n(s^\prime,a,s;t)$ as the number of class-$n$ arms which are activated under the policy $\pi_n^\star$ and transit from state $s^\prime$ at time $t$ to state $s$ at time $t+1$, and $G_n(s^\prime,0,s;t)$ as the number of class-$n$ arms which are kept passive under the policy $\pi_n^\star$ and transit from state $s^\prime$ at time $t$ to state $s$ at time $t+1$. Hence we have 
\begin{align*}
    B_n(s;t+1)=\sum_{s^\prime,a\in\mathcal{A}\setminus\{0\}}C_n(s^\prime,a,s;t)+G_n(s^\prime,0,s;t).
\end{align*}
Dividing both sides by $\rho$ yields
\begin{align*}\label{eq3:lemma3}
    \lim_{\rho\rightarrow\infty} \frac{B_n(s;t+1)}{\rho}=&\lim_{\rho\rightarrow\infty}\sum_{s^\prime,a\in\mathcal{A}\setminus\{0\}} \frac{C_n(s^\prime,a, s;t)}{\rho}
    +\lim_{\rho\rightarrow\infty}\sum_{s^\prime} \frac{G_n(s^\prime,0, s;t)}{\rho}.
\end{align*}
Note that $C_n(s^\prime, a,s;t)$ is a binomial random variable with $D_n(s^\prime,a; t)$ trails and success probability $P_n(s^\prime,a,s)$. Similarly, $G_n(s^\prime, 0,s;t)$ is a binomial random variable with $B_n(s^\prime;t)-\sum_{a\in\mathcal{A}\setminus\{0\}}D_n(s^\prime,a;t)$ trials and success probability $P_n(s^\prime,0,s)$. Then, we can rewrite the above equation as follows 
\begin{align*}
   &\lim_{\rho\rightarrow\infty} \frac{B_n(s;t+1)}{\rho}\\
=&\sum_{s^\prime,a\in\mathcal{A}\setminus\{0\}}\lim_{\rho\rightarrow\infty} \frac{D_n(s^\prime,a;t)}{\rho}P_n(s^\prime,a,s)+\sum_{s^\prime}\lim_{\rho\rightarrow\infty} \frac{B_n(s^\prime;t)-\sum_{a\in\mathcal{A}\setminus\{0\}}D_n(s^\prime,a;t)}{\rho}P(s^\prime,0,s)\displaybreak[1]\\
=&\sum_{s^\prime,a\in\mathcal{A}\setminus\{0\}}P_n(s^\prime;t)\chi_n^\star(s^\prime,a;t)P_n(s^\prime,a,s)+\sum_{s^\prime}P_n(s^\prime;t)\left(1-\sum_{a\in\mathcal{A}\setminus\{0\}}\chi_n^\star(s^\prime,a;t)\right)P_n(s^\prime,0,s)\displaybreak[3]\\
=&P_n(s,t+1). ~~a.s. 
\end{align*}

Next we show that the second equation in Lemma~\ref{lemma1} holds for time $t+1.$
To ease the notation, we first define the set $\cI_n(s;t)$ that contains all arms at states that have a higher index than the index of class-$n$ arms at state $s$ at time $t$, i.e.,
\begin{align*}
    \cI_n(s;t):=\Big\{(i,j)\big\vert\psi_i^\star(j;t)>\psi_n^\star(s;t),
    \forall i\in[N], j\in\cS\Big\}.
\end{align*}
To ease the expression, we define the resources consumed before activating arm $n$ at state $s$ at time $t$ as
\begin{align*}
    K_{n,s;t+1}:=\sum_{(i,j)\in\cI_n(s;t)}\!\!\!B_i(j,t+1)\!\!\!\sum_{a\in\mathcal{A}\setminus\{0\}}a\chi_i^\star(j,a;t+1).
\end{align*}
Then, based on our \textit{OMR Index Policy}, we have
\begin{align*}
&\sum_{a\in\mathcal{A}\setminus\{0\}}aD_n(s,a;t+1)\\
=&\min\Bigg(\Big(\rho K-K_{n,s;t+1}\Big)^+, B_n(s;t+1)\sum_{a\in\mathcal{A}\setminus\{0\}}a\chi_n^\star(s,a;t+1)\Bigg)\\  
=&\pmb{1}\Bigg(\rho K-K_{n,s;t+1}\geq B_n(s,t+1)\sum_{a\in\mathcal{A}\setminus\{0\}}a\chi_n^\star(s,a;t+1)\Bigg)\cdot B_n(s,t+1)\sum_{a\in\mathcal{A}\setminus\{0\}}a\chi_n^\star(s,a;t+1)\\
+&\pmb{1}\Bigg(0\leq\rho K-K_{n,s;t+1}< B_n(s,t+1)\sum_{a\in\mathcal{A}\setminus\{0\}}a\chi_n^\star(s,a;t+1)\Bigg)\cdot\Big(\rho K-K_{n,s;t+1}\Big).
\end{align*}
Dividing both sides by $\rho$ and taking the limit, we have
\begin{align*}
    &\lim_{\rho\rightarrow\infty}\sum_{a\in\mathcal{A}\setminus\{0\}}\frac{aD_n(s,a;t+1)}{\rho}\displaybreak[0]\\
=&\pmb{1}\Bigg(K-K_{n,s;t+1}/\rho\geq P_n(s;t+1)\sum_{a\in\mathcal{A}\setminus\{0\}}a\chi_n^\star(s,a;t+1)\Bigg)\cdot P_n(s;t+1)\sum_{a\in\mathcal{A}\setminus\{0\}}a\chi_n^\star(s,a;t)\displaybreak[2]\\
+&\pmb{1}\Bigg(0\leq K-K_{n,s;t+1}/\rho< P_n(s;t+1)\sum_{a\in\mathcal{A}\setminus\{0\}}a\chi_n^\star(s,a;t+1)\Bigg)\cdot\Big( K-K_{n,s;t+1}/\rho\Big).
\end{align*}

In the following, we prove the desired results by considering three cases. First, we assume that all arms of class-$n$ at state $s$ at time $t+1$ cannot be activated, i.e., $\chi_n^\star(s,a;t+1)=0, \forall a\in \mathcal{A}\setminus\{0\}$, which implies that $K-K_{n,s;t+1}/\rho\leq 0$. Hence we have
\begin{align*}
   & \lim_{\rho\rightarrow\infty}\frac{D_n(s,a;t+1)}{\rho}=0
    =P_n(s;t+1)\chi_n^\star(s,a;t+1).
\end{align*}
Next, we assume that all arms of class-$n$ at state $s$ at time $t+1$ can be activated, which means $K-K_{n,s;t+1}/\rho\geq P_n(s,t+1)\sum_{a\in\mathcal{A}\setminus\{0\}}a\chi_n^\star(s,a;t).$
In this case,  Since the actions are randomly selected according to $\chi_n^\star(s,a;t+1)$, we have 
\begin{align*}
 \lim_{\rho\rightarrow\infty}\frac{D_n(s,a;t+1)}{\rho}=P_n(s;t+1)\chi_n^\star(s,a;t+1).   
\end{align*}
Last, we assume that only partial arms of class-$n$ at state $s$ at time $t+1$ can be activated, which implies $0<\sum_{a\in\mathcal{A}\setminus\{0\}}\chi_n^\star(s,a;t+1)<1$.   Due to the activation budget constraint, we have 
\begin{align*}
    0<K-K_{n,s;t+1}/\rho= P_n(s,t+1)\!\!\!\sum_{a\in\mathcal{A}\setminus\{0\}}\!\!\!a\chi_n^\star(s,a;t+1).
\end{align*}
Based on our \textit{OMR Index Policy}, we have the following 
\begin{align*}
    &\lim_{\rho\rightarrow\infty}\sum_{a\in\mathcal{A}\setminus\{0\}}\frac{aD_n(s,a;t+1)}{\rho}\\
    =&K-\sum_{(i,j)\in\cI_n(s;t)}\!\!\!P_i(j;t+1)\!\!\!\sum_{a\in\mathcal{A}\setminus\{0\}}a\chi_i^\star(j,a;t+1)\\
    =& P_n(s;t+1)\sum_{a\in\mathcal{A}\setminus\{0\}}ad_n(s,a;t+1).
\end{align*}
According to the law of large number, we obtain
\begin{align*}
    \lim_{\rho\rightarrow\infty} d_n(s,a;t+1)=\chi_n^\star(s,a;t+1),  a.s.
\end{align*}
Therefore, we have 
\begin{align*}
    &\lim_{\rho\rightarrow\infty}\frac{D_n(s,a;t+1)}{\rho}
    = P_n(s;t+1)\chi_n^\star(s,a;t+1), 
\end{align*}
which completes the proof.
\end{proof}

Now we are ready to present the proof of Theorem~\ref{thm:asym_opt}. 
On the one hand, it is clear that the total reward our \textit{OMR Index Policy} $\pi^\star, $ achieves cannot exceed that achieved by the optimal policy, i.e., 
\begin{align*}
    \lim_{\rho\rightarrow\infty}\frac{R(\pi^\star,\rho K, \rho N )}{\rho }\leq \lim_{\rho\rightarrow\infty}\frac{R(\pi^{opt},\rho K, \rho N)}{\rho }.
\end{align*}
On the other hand, we show that the total reward obtained by our \textit{OMR Index Policy} is not less than that achieved by the optimal policy based on Lemma~\ref{lemma1}, i.e., 
\begin{align*}
    &\lim_{\rho\rightarrow\infty}\frac{R(\pi^\star,\rho K, \rho N)}{\rho}\\
=&\lim_{\rho\rightarrow\infty}\frac{1}{\rho }\mathbb{E}_{\pi^\star}\Bigg[\sum_{n=1}^{N}\sum_{t=1}^{T}\sum_{s,a\in\mathcal{A}\setminus\{0\}}\bar{r}_n(s,a)D_n(s,a;t)\\
&\qquad\qquad\qquad\qquad\qquad\qquad+\sum_{n=1}^{N}\sum_{t=1}^{T}\sum_{s}\bar{r}_n(s,0)(B_n(s,t)-\sum_{a\in\mathcal{A}\setminus\{0\}}D_n(s,a;t))\Bigg]\displaybreak[0]\\
=&\sum_{n=1}^{N}\sum_{t=1}^{T}\sum_{s, a\in\mathcal{A}\setminus\{0\}}\bar{r}_n(s,a)\lim_{\rho\rightarrow\infty}\frac{\mathbb{E}_{\pi^\star}[D_n(s,a;t)]}{\rho}\displaybreak[1]\\
&\qquad\qquad\qquad\qquad\qquad\qquad+\sum_{n=1}^{N}\sum_{t=1}^{T}\sum_{s}\bar{r}_n(s,0)\lim_{\rho\rightarrow\infty}\frac{\mathbb{E}_{\pi^\star}\left[B_n(s,t)-\sum\limits_{a\in\mathcal{A}\setminus\{0\}}D_n(s,t)\right]}{\rho}\displaybreak[2]\\
\overset{(a)}{=}& \sum_{n=1}^{N}\sum_{t=1}^{T}\sum_{s, a\in\mathcal{A}\setminus\{0\}}\bar{r}_n(s,a)P_n(s;t)\chi_n^\star(s,a;t)+ \sum_{n=1}^{N}\sum_{t=1}^{T}\sum_{s}\bar{r}_n(s,0)P_n(s,t)\chi_n^\star(s,0;t)\\ 
=&\sum_{n=1}^{N}\sum_{t=1}^{T}\sum_{s, a\in\mathcal{A}\setminus\{0\}}\bar{r}_n(s,a)\mu_n^\star(s,a;t)+ \sum_{n=1}^{N}\sum_{t=1}^{T}\sum_{s}\bar{r}_n(s,0)\mu_n^\star(s,0;t)\\ 
\overset{(b)}{\geq}& \lim_{\rho\rightarrow\infty}\frac{R(\pi^{opt},\rho K, \rho N)}{\rho },
\end{align*}
where (a) follows from Lemma \ref{lemma1} and (b) comes from the fact that the relaxed problem achieves an upper bound of the original optimal solution. This completes the proof.

\section{Proof of Theorem~\ref{thm:regret}}
In this section, we present the proof detail of Theorem~\ref{thm:regret}, i.e., the regret of R(MA)$^2$B-UCB.  Since there are two phases in R(MA)$^2$B-UCB,  we decompose the regret $\Delta(\pi^\star,\bs_1,T)$ into two distinct parts, i.e., the regret of the planning phase with any random policy and the regret of the policy execution phase.  To this end, we divide the total time horizon $T$ into the planning part $T_1$ and the policy execution part $T_2$, respectively, i.e., $T=T_1+T_2$. Then the total regret can be expressed as  
\beq\nonumber
\Delta(\pi^\star,\bs_1; T)=\Delta(T_1)+\Delta(\pi^\star, \bs_1, T_2).
\eeq 
In the following, we derive the regrets for these two parts, respectively.

\subsection{The Regret of the Planning Phase}
{ In the planning phase, each state-action pair $(s,a)$ for each arm is uniformly sampled for $\Lambda(T)$ times according to a generative model}. The performance gap between the optimal policy and the random policy at each time is bounded since the reward is bounded and { the maximum number of arms can be activated at each time is not larger than $K$}.  To this end, we can easily bound the regret of the planning phase $\Delta(T_1)$ as presented in the following lemma.

\begin{lemma}\label{lemma2}
Since the reward is bounded and not greater than one, the regret in the planning phase can be bounded as
$$
\Delta(T_1)=\cO\left({SAK}\cdot \Lambda(T)\right).
$$
\end{lemma}

\begin{proof}
The result directly follows from the subsequent two facts. First,  there are $N$ arms with a total number of state-action pairs $SA$ and to guarantee each state-action pair being sampled for $\Lambda(T)$ times, it requires $SA\cdot\Lambda(T)$ time slots since the agent can observe all arms' state-action pairs at each time slot. Second, at each decision time, {the maximum number of arms that can be activated is not greater than $K$ due to the budget constraint.}
\end{proof}

\subsection{The Regret of the Policy Execution Phase}
We next analyze the regret of the policy execution phase, i.e., $\Delta(\pi^\star, \bs_1, T_2)$, which is defined as 
\begin{align*}
    \Delta(\pi^\star, \bs_1, T_2):=\mathbb{E}[R(\pi^{opt},\bs_1, T_2)]-\mathbb{E}[R(\pi^\star, \bs_1, T_2)],
\end{align*}
which characterizes the accumulated reward gap when the optimal policy $\pi^{opt}$ and the learned policy $\pi^\star$ are executed, respectively. 
 For the entire parameter space,  two possible and disjoint events can occur at the policy execution phase. 
 The first event is called \textit{the failure event}, which occurs when the true MDPs $\{M_n\}$ lie outside the plausible MDPs set $\cM$ that we construct in line 4 of the R(MA)$^2$B-UCB policy.  The second event is called \textit{the good event} when the true MDPs $\{M_n\}$ lie inside the plausible MDPs set $\cM$.  Therefore, the regret of the policy execution phase can be decomposed into two parts as follows
\begin{align*}
    \Delta(\pi^\star, \bs_1, T_2)=&\Delta(\pi^\star, \bs_1, T_2)\pmb{1}(\{{M}_n\}\notin\cM)+\Delta(\pi^\star, \bs_1, T_2)\pmb{1}(\{{M}_n\}\in\cM).
\end{align*}

\subsubsection{Regret Conditioned on the Failure Event}
Specifically, we define the failure events as follows:
\begin{align*}
    \cE_p:=\{\exists (s, a), n, |{P}_n(s^\prime|s,a)-\hat{P}_n(s^\prime|s,a)|> \delta_n(s,a) \}, 
\end{align*}
and 
\begin{align*}
 \cE_r:=\{\exists (s, a), n, |\bar{r}_n(s,a)-\hat{r}_n(s,a)|> \delta_n(s,a)\},
\end{align*}
which indicate that the true parameters are outside the confidence intervals constructed in (9). 
We denote the correspondingly complementary events as $\cE_p^c$ and $\cE_r^c$, respectively. Therefore, we have  
$\{{M}_n\}\notin\cM:=\cE_p\cup\cE_r,$ $\{{M}_n\}\in\cM:=\cE_p^c\cap\cE_r^c.$
Given these, we first characterize the probability that the failure event occurs. 

\begin{lemma}\label{lemma3}
Provided that $\delta_n(s,a)=\sqrt{\frac{1}{2\Lambda(T)}\log\Big(\frac{2SA N\Lambda(T)}{\eta}\Big)}$, we have 
\begin{align*}
    Pr(\{{M}_n\}\notin\cM)\leq \frac{2\eta}{\Lambda(T)}.
\end{align*}
\end{lemma}
\begin{proof}
By Chernoff-Hoeffding inequality \cite{hoeffding1994probability}, we have
\begin{align*}
    Pr\big(|{P}_n(s^\prime|s,a)-\hat{P}_n(s^\prime|s,a)|> \delta_n(s,a) \big)
\leq \frac{\eta}{SA N\Lambda(T)}.
\end{align*}
By leveraging the union bound over all states, actions and number of arms, we have  
\begin{align*}
   Pr(\{{M}_n\}\notin\cM)
   \leq&\sum_{n=1}^N\sum_{(s,a)}Pr\big(|{P}_n(s^\prime|s,a)-\hat{P}_n(s^\prime|s,a)|> \delta_n(s,a)\big)
   \\
   &\qquad+\sum_{n=1}^N\sum_{(s,a)}Pr\big(|\bar{r}_n(s,a)-\hat{r}_n(s,a)|> \delta_n(s,a)\big)\\
    \leq& \frac{2\eta}{\Lambda(T)}.
\end{align*}
\end{proof}
Using Lemma~\ref{lemma3}, we characterize the regret conditioned on the failure event.  
\begin{lemma}\label{lemma4}
The regret conditioned on the failure event is given by 
\begin{align*}
    \Delta(\pi^\star, \bs_1, T_2)\pmb{1}(\{{M}_n\}\notin\cM)\leq \frac{2KT_2\eta}{\Lambda(T)}.
\end{align*}
\end{lemma}

\begin{proof}
According to Lemma \ref{lemma3}, we have 
\begin{align*}
   \Delta(\pi^\star, \bs_1, T_2)\pmb{1}(\{{M}_n\}\notin\cM)
&\leq KT_2 \pmb{1}(\{{M}_n\}\notin\cM)\leq\frac{2KT_2\eta}{\Lambda(T)}, 
\end{align*}
where the first inequality comes from the fact that at each time slot the regret is upper bounded by $K$.
\end{proof}

\subsubsection{Regrets Conditioned on the Good Event}
Provided Lemma \ref{lemma3}, we have that the probability that the true MDP is inside the plausible MDPs set, i.e., $\{{M}_n\}\in\cM$,  is at least $1-\frac{2\eta}{\Lambda(T)}$. 
Now we consider the regret conditioned on the good event $\{{M}_n\}\in\cM$. Define $\xi^{opt}$ as the optimal average reward achieved by the optimal policy  $\pi^{opt}$ and $\xi^{\star}$ as optimistic average reward achieved by the learned policy $\pi^\star$ for the ture MDP ${M_n}$. Then we have
\begin{align*}
    \Delta(\pi^\star, \bs_1, T_2)=T_2\xi^{opt}-T_2\xi^{\star}.
\end{align*}
We first present a key lemma.

\begin{lemma} (Optimism)
Conditioned on the good event, there exists a transition $\tilde{P}_n\in\cP_n, \forall n\in[N]$ such that 
\begin{align*}
    &\sum_{n=1}^{N}\sum_{t=1}^{T_2}\sum_{(s,a)}\mu_n^{opt}(s,a;t)\bar{r}_n(s,a)
    \leq \sum_{n=1}^{N}\sum_{t=1}^{T_2}\sum_{(s,a)}\tilde{\mu}_n(s,a;t)\tilde{r}_n(s,a),
\end{align*}
where $\tilde{\mu}_n$ is the optimal occupancy measure derived from $\{\tilde{P}_n, \forall n\in[N]\}.$
\label{lemma5}
\end{lemma}
\begin{proof}
Conditioning on the good event, the true model $P_n$ is contained in $\cP_n, \forall n$.
Furthermore, conditioned
on the good event $\bar{r}_n(s,a)\leq \tilde{r}_n(s,a)$.  By setting $\tilde{P}_n=P_n$, we have
\begin{align*}
\sum_{n=1}^{N}\sum_{t=1}^{T_2}\sum_{(s,a)}\mu_n^{opt}(s,a;t)\bar{r}_n(s,a)
\leq&\sum_{n=1}^{N}\sum_{t=1}^{T_2}\sum_{(s,a)}\mu_n^{opt}(s,a;t)\tilde{r}_n(s,a)\\ 
= &\sum_{n=1}^{N}\sum_{t=1}^{T_2}\sum_{(s,a)}\tilde{\mu}_n(s,a;t)\tilde{r}_n(s,a),
\end{align*}
which completes the proof.
\end{proof}

\begin{remark}
Lemma \ref{lemma5} indicates that inside the plausible MDPs set $\cM$, there exists an MDP $\{\tilde{M}_n\}$ with parameters $\{\tilde{P}_n, \tilde{r}_n\}$ achieving no less accumulated reward compared to the reward achieved by the optimal policy for the true MDP $\{M_n\}$. This fact is summarized in the following lemma.
\end{remark}

\begin{lemma}\label{lemma6}
There exists an optimistic MDP $\{\tilde{M}_n\}\in\cM$ such that {the associated policy} $\tilde{\pi}$ can achieve a total reward no less than that achieved by the optimal policy, i.e.,
\begin{align*}
    T_2\xi^{opt}\leq T_2\tilde{\xi},
\end{align*}
where $\tilde{\xi}$ is the average reward per time slot for the optimistic MDP $\{\tilde{M}_n\}$ under the policy $\tilde{\pi}.$
\end{lemma}
\begin{proof}
The following relation holds
\begin{align*}
    T_2\xi^{opt}&\overset{(a)}{=}\sum_{n=1}^{N}\sum_{t=1}^{T_2}\sum_{(s,a)}\mu_n^{opt}(s,a;t)\bar{r}_n(s,a)\displaybreak[0]\\
    &\overset{(b)}{\leq} \sum_{n=1}^{N}\sum_{t=1}^{T_2}\sum_{(s,a)}\mu_n^{opt}(s,a;t)\tilde{r}_n(s,a)\displaybreak[1]\\
    &\overset{(c)}{\leq} \max_{\{\mu_n\}, \{\tilde{M}_n\}\in\cM}\sum_{n=1}^{N}\sum_{t=1}^{T_2}\sum_{(s,a)}\mu_n(s,a;t)\tilde{r}_n(s,a)\displaybreak[2]\\ 
    &:=\sum_{n=1}^{N}\sum_{t=1}^{T_2}\sum_{(s,a)}\tilde{\mu}_n(s,a;t)\tilde{r}_n(s,a)=T_2\tilde{\xi},
\end{align*}
where (a) holds true since our \textit{OMR Index Policy} achieves the optimality, (b) is due to the fact that $\bar{r}_n(s,a)\leq \tilde{r}_n(s,a), \forall s, a, n$, and (c) follows directly from Lemma \ref{lemma5}.
\end{proof}

We are now ready to characterize the regret conditioned on the good event. 

{
\begin{lemma}\label{lemma7}
Conditioned on the good event, the regret is given by 
\begin{align*}
    \Delta(\pi^\star, \bs_1, T_2)=\cO(2K\sqrt{T}).
\end{align*}
\end{lemma}
\begin{proof}
Based on Lemma~\ref{lemma6}, there exists a policy $\pi^\prime$ which achieves an average reward $\xi^\prime$ for MDP $\{M_n^\prime\}$ such that
\begin{align*}
  \Delta(\pi^\star, \bs_1, T_2)&=T_2\xi^{opt}-T_2\xi^{\star}\leq T_2\xi^\prime-T_2\xi^{\star}.  
\end{align*}
{ Without loss of generality, we assume that the policy $\pi^\prime$ satisfies 
$$T_2{\xi^\prime}=\sum_{n=1}^{N}\sum_{t=1}^{T_2}\sum_{(s,a)}{\mu}^{opt}_n(s,a;t)(\bar{r}_n(s,a)+\sigma)\geq T_2\xi^{opt},$$
}
with $\sigma\geq 0$.
Under the policy $\pi^\star$, we define
\begin{align*}
    T_2\xi^\star:=\sum_{n=1}^{N}\sum_{t=1}^{T_2}\sum_{(s,a)}{\mu}_n^\star(s,a;t)\bar{r}_n(s,a).
\end{align*}
Hence we have 
\begin{align*}
\Delta(\pi^\star, \bs_1, T_2)
\leq&\sum_{n=1}^{N}\sum_{t=1}^{T_2}\sum_{(s,a)}{\mu}^{opt}_n(s,a;t)(\bar{r}_n(s,a)+\sigma)-\sum_{n=1}^{N}\sum_{t=1}^{T_2}\sum_{(s,a)}{\mu}_n^\star(s,a;t)\bar{r}_n(s,a)\displaybreak[1]\\
\overset{(a)}{=}& \underset{term1}{\underbrace{\sum_{n=1}^{N}\sum_{t=1}^{T_2}\sum_{(s,a\in\mathcal{A}\setminus \{0\})}({\mu}^{opt}_n(s,a;t)-\mu_n^\star(s,a;t))\bar{r}_n(s,a)}}\\
&\qquad+\underset{term2}{\underbrace{\sum_{n=1}^{N}\sum_{t=1}^{T_2}\sum_{(s,a\in\mathcal{A}\setminus \{0\})}{\mu}^{opt}_n(s,a;t)\sigma}}\\ 
\overset{(b)}{\leq}& \cO(2K\sqrt{T}),
\end{align*}
{where $(a)$ holds due to the fact that $\bar{r}_n(s,0)=0$. (b) follows from that (i) we have $term 1\leq 0$ based on optimism that the $\sum_{n=1}^{N}\sum_{t=1}^{T_2}\sum_{(s,a)}{\mu}^{\star}_n(s,a;t)\bar{r}_n(s,a)$ is no less than $\sum_{n=1}^{N}\sum_{t=1}^{T_2}\sum_{(s,a)}{\mu}^{opt}_n(s,a;t)\bar{r}_n(s,a)$;  
and (ii) we have $term2 \leq KT_2\sigma$ since $\mu^{opt}_n(s,a;t)$ is the occupancy measure and satisfies the budget constraint. By letting $\sigma=\frac{2}{\Lambda(T)}$, we have the result as shown in (b). This completes the proof.} 
\end{proof}}

\subsection{Total Regret}
According to Lemma \ref{lemma2}, Lemma \ref{lemma4} and Lemma \ref{lemma7}, when $\Lambda(T)=\sqrt{T}$,
the total regret is given by 
\begin{align*}
\Delta(\pi^\star,\bs_1; T)
=&\Delta(T_1)+\Delta(\pi^\star, \bs_1, T_2)\\  =&\cO\left({SAK} \Lambda(T)\right)+\cO({2K\eta\sqrt{T}})+\cO(2K\sqrt{T})\\ 
=&\cO((SAK+2K(1+\eta))\sqrt{T}).
\end{align*}
This completes the proof of Theorem~\ref{thm:regret}.

\end{document}